\documentclass[11pt]{article}

\PassOptionsToPackage{table}{xcolor}
\usepackage[preprint]{acl}
\usepackage{times}
\usepackage{latexsym}
\usepackage[T1]{fontenc}
\usepackage[utf8]{inputenc}
\usepackage{microtype}

\usepackage{amsmath,amssymb,amsthm,mathtools}
\newtheorem{proposition}{Proposition}
\usepackage{etoolbox}
\BeforeBeginEnvironment{equation}{\begingroup\small\setlength{\abovedisplayskip}{6pt}\setlength{\belowdisplayskip}{6pt}}
\AfterEndEnvironment{equation}{\endgroup}
\makeatletter
\pretocmd{\section}{\vspace{-3pt}}{}{}
\pretocmd{\subsection}{\vspace{-3pt}}{}{}
\pretocmd{\subsubsection}{\vspace{-2pt}}{}{}
\makeatother

\usepackage{graphicx}
\usepackage{booktabs,array,tabularx,multirow,makecell,colortbl}
\usepackage{subcaption}
\usepackage{placeins}
\usepackage{float}
\usepackage{enumitem}
\usepackage{pifont}   
\usepackage{soul}     
\usepackage{url}

\newfloat{algorithm}{tbp}{loa}
\floatname{algorithm}{Algorithm}
\floatstyle{ruled}\restylefloat{algorithm}

\definecolor{GainGreen}{HTML}{21804C}
\definecolor{LossRed}{HTML}{A33F3F}
\definecolor{SoftGray}{HTML}{5A626A}

\newcommand{\shline}{\Xhline{0.8pt}}
\newcolumntype{L}[1]{>{\raggedright\arraybackslash}p{#1}}
\newcolumntype{C}{>{\centering\arraybackslash}X}
\newcolumntype{Y}{>{\raggedright\arraybackslash}X}

\newcommand{\gainsub}[1]{\textsubscript{\normalfont\fontsize{6}{6}\selectfont\color{GainGreen}$+#1$}}
\newcommand{\losssub}[1]{\textsubscript{\normalfont\fontsize{6}{6}\selectfont\color{LossRed}$-#1$}}

\newcommand{\method}{\textsc{GlitchPatch}}

\renewcommand{\arraystretch}{1.0}
\makeatletter
\renewcommand\paragraph{\@startsection{paragraph}{4}{\z@}{1.3ex plus .4ex minus .2ex}{-1em}{\normalfont\normalsize\bfseries}}
\makeatother
\setlist[itemize]{leftmargin=1.4em,itemsep=1pt,topsep=2pt,parsep=0pt}
\setlist[enumerate]{leftmargin=1.6em,itemsep=1pt,topsep=2pt,parsep=0pt}

\title{GlitchPatch: Repairing Glitch Tokens in Frozen Language Models\\via Local Retokenization}

\usepackage{fancyhdr}
\author{%
  \normalsize
  Kunsheng Tang\textsuperscript{1}\thanks{Equal contribution.}\enspace
  Peigui Qi\textsuperscript{1}\footnotemark[1]\enspace
  Yide Song\textsuperscript{2}\enspace
  Peijun Huang\textsuperscript{3}\enspace
  Weiming Zhang\textsuperscript{1}\enspace
  Nenghai Yu\textsuperscript{1} \\[3pt]
  \normalfont\normalsize
  \textsuperscript{1}University of Science and Technology of China\enspace
  \textsuperscript{2}University of Washington\enspace
  \textsuperscript{3}Wuhan University
}

\begin{document}
\maketitle

\begin{abstract}
Glitch tokens are anomalous vocabulary entries that make large language
models (LLMs) produce outputs inconsistent with their inputs. Existing
repairs modify model internals and are therefore impractical for frozen
checkpoints. We ask whether glitch tokens can be repaired from outside
the model by optimizing the input tokenization. An empirical study shows
that deleting a glitch token's BPE merge rule fixes many failures but
damages normal tokens that share merge nodes, and that no decomposition
depth is uniformly best. We therefore propose \textbf{\method{}}, which
offline searches, per glitch token, for the behaviorally optimal
equivalent token sequence (Behavioral Path Optimization), and online
substitutes only the IDs of matched glitch tokens, leaving the tokenizer
and model untouched. On ten models from six tokenizer families,
\method{} fixes 85.10\% of glitch tokens, 14.37 points above the
strongest baseline, with zero regression on normal tokens and
negligible latency. We hope this work provides a practical option for
improving tokenizer reliability.
\end{abstract}

\section{Introduction}
\label{sec:introduction}

The vocabularies of large language models (LLMs) contain anomalous
entries known as \emph{glitch tokens}, which can induce outputs
inconsistent with their inputs~\citep{land2024fishing}. Some glitch
tokens closely resemble normal tokens, making them easy to trigger
inadvertently and leading to unexpected behavior that degrades user experience (Figure~\ref{fig:introduction}). The problem persists in current LLMs such as GPT-6~\citep{openaicommunity2026gpt6}, Claude Opus~5~\citep{hruss2026opus5}, and GPT-5~\citep{finke2025gptoss}. When such strings appear
in product names, database fields, or code identifiers, errors can
propagate to downstream tasks. Prior work has linked glitch
tokens to unsafe generation and degraded instruction
following~\citep{land2024fishing,li2024glitch,tang2026bad}, making their repair a
necessary step toward improving model reliability.

\begin{figure*}[t]
\centering
\includegraphics[width=0.88\textwidth]{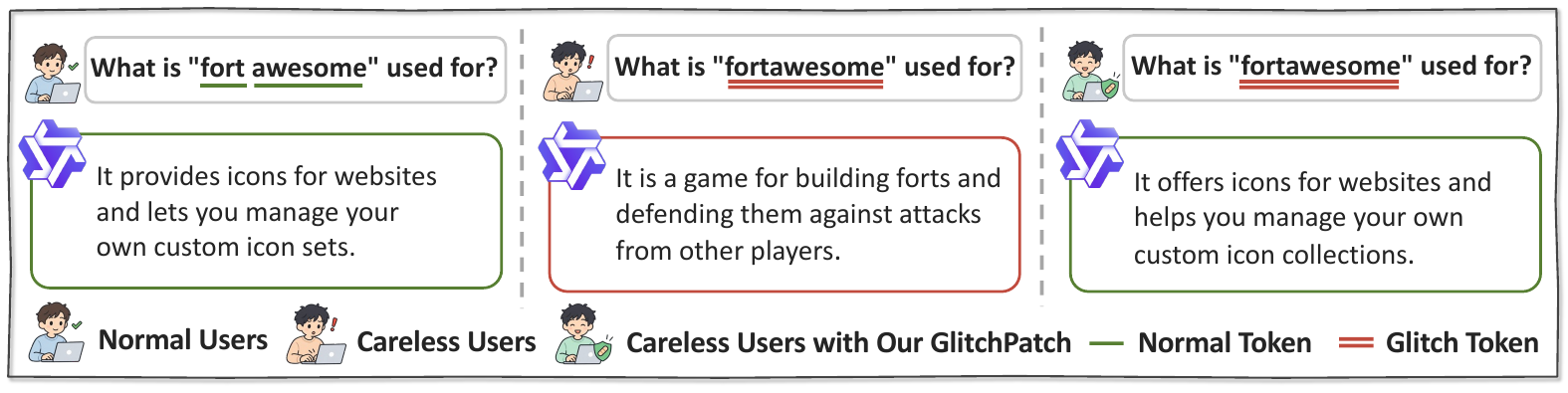}
\caption{\textbf{A glitch-token example and our \method{} repair.} The strings ``fortawesome'' (glitch token) and ``fort awesome'' (normal tokens) differ by one space. The model describes an icon library for the latter but a fort-building game for the former. \method{} restores the relevant answer without changing the input text.}
\label{fig:introduction}
\end{figure*}

Existing work repairs glitch tokens from two directions.
Activation-based methods, represented by
GlitchProber~\citep{zhang2024glitchprober}, access intermediate-layer activations at inference time and correct target neuron states. Parameter-based methods modify the model's parameters:
GlitchEdit~\citep{tang2026bad} edits the embedding weights of target tokens, and GlitchCleaner~\citep{fan2026glitchcleaner} trains and attaches a gated LoRA branch. These methods are effective in their respective settings, but they must happen \textbf{inside} the model. Whether inserting activation hooks, rewriting embedding weights, or loading adapter parameters, all require access to model internals or alter the model's parameters. In practice, however, most providers deploy frozen checkpoints
integrated into production pipelines with verified kernels, version control, and regression testing. Any internal modification can invalidate these pipelines and introduce regression risks. When providers discover glitch tokens, they typically need a fast, reversible fix that leaves the model itself untouched rather than a costly internal refactor. \textbf{To date, no
work has explored whether glitch tokens can be repaired from outside the model.}

Tokenization provides a natural entry point for external repair. An LLM
does not process raw text directly; its tokenizer encodes the input into
a sequence of vocabulary tokens, each mapped to a learned embedding (Figure~\ref{fig:bpe}). Tokenizers typically use byte-pair encoding (BPE), which builds tokens
by progressively merging shorter units according to learned merge rules. If the final token is a glitch token, an
intuitive fix is to block its last merge so that the model receives the
normal sub-tokens instead. The observation that frozen models can
tolerate non-canonical tokenizations supports this
direction~\citep{zheng2025broken}. However, this intuition does not
directly yield reliable repair. We conduct an empirical study on deleting BPE glitch token merge rules and uncover two key
findings (Figure~\ref{fig:deletion}). First, deleting a glitch token's
merge rule does fix a portion of failures but also disrupts normal longer tokens that share the same intermediate merge node (Figure~\ref{fig:shared}a), causing the overall glitch rate to increase (\textbf{Finding~1}). Second,
decomposing a glitch token to different granularities produces replacement paths whose
fix rates vary non-monotonically across models
(Figure~\ref{fig:shared}b), while collateral damage to normal tokens
grows monotonically with depth (\textbf{Finding~2}). These two findings expose two open problems for external repair:  \textbf{how to find the optimal replacement path for each glitch token, and how to eliminate collateral effects on normal tokens.}

Motivated by these findings, we propose \textbf{\method{}}, an external repair
framework for frozen language models based on Local Retokenization
(Figure~\ref{fig:framework}). \method{} leaves the tokenizer unchanged
and replaces only the IDs of matched glitch tokens in the canonical
token sequence. The framework has two stages: \textbf{offline}, Behavioral Path
Optimization (BPO) searches for the behaviorally optimal equivalent
token sequence for each known glitch token and compiles validated
replacements into a rule table through full-context evaluation and
independent calibration; \textbf{online}, only rule matching and ID substitution
are performed, without reading or modifying any model parameters or
internal states. Experiments on ten models spanning six tokenizer families show that:
\ding{182}~for repair effectiveness, \method{} achieves an 85.10\% mean fix
rate, outperforming the strongest baseline by 14.37 percentage points
and reducing the mean glitch rate from 14.88\% to 2.27\%;
\ding{183}~for impact on normal tokens, the regression rate is 0.00\%
on the full-vocabulary evaluation, while requests that trigger no
applicable rule are preserved by construction;
\ding{184}~for practicality, the byte-normalized loss on natural texts containing glitch tokens drops by 8.4\% on average, and the identical continuation following a glitch token becomes 3.8\% more predictable, with negligible online latency
overhead.
These results demonstrate that by optimizing the tokenization, an external method can not only repair glitch tokens
effectively but outperform existing internal approaches.

In summary, our contributions are:
\begin{itemize}[leftmargin=*,nosep]
  \item We present the first exploration of repairing glitch tokens from outside the model, identifying two limitations of merge-rule deletion that motivate our approach (Section~\ref{sec:empirical}).
  \item We propose \method{}, a framework that combines offline BPO path optimization with online local retokenization, repairing glitch tokens with minimal impact on normal tokens and overhead (Section~\ref{sec:method}).
  \item Extensive experiments show that \method{} achieves a mean fix rate of 85.10\%, surpassing all existing methods, with validated practical deployment feasibility in time cost, language modeling, and general capability (Sections~\ref{sec:setup} and~\ref{sec:results}).
\end{itemize}

\section{Background and Related Work}
\label{sec:background}

This section introduces tokenization in LLMs and reviews glitch-token
detection and repair, clarifying the gap between existing methods and
frozen deployment needs.

\subsection{Tokenizer and Tokenization}

LLMs do not process raw text directly. The input is
first split by a tokenizer into vocabulary tokens, each mapped to a
vector via an embedding lookup and then processed by the Transformer. A
token may be a word fragment or part of a byte
sequence, so different segmentations of the same text produce different
model inputs.

Most modern tokenizers adopt Byte Pair Encoding (BPE). BPE learns an
ordered set of merge rules from the training corpus; at encoding time,
it starts from the smallest units and repeatedly merges adjacent pairs
by rule priority until no further merge applies
(Figure~\ref{fig:bpe}, Qwen3.8-27B~\citep{qwen2026}). For a given text and tokenizer, this
deterministic result is the canonical tokenization, denoted
$T(x)$. However, blocking a merge step or using other vocabulary
entries covering the same byte span can produce ID sequences that decode
to the same bytes; we call these equivalent tokenizations. Prior work has shown
that some frozen models tolerate or even benefit from non-canonical
segmentations~\citep{zheng2025broken,ghosh2026robust,DBLP:conf/emnlp/SawadaG25,DBLP:conf/naacl/SatheAS25,DBLP:conf/emnlp/JangLCPS25}, and techniques
such as BPE-Dropout enhance this robustness
by introducing segmentation variation during
training~\citep{kudo2018subword,provilkov2020bpe,DBLP:conf/emnlp/ChaiF0024}. These findings
indicate that tokenization is a manipulable variable, providing the
basis for repairing glitch tokens at the tokenization level.

\subsection{Glitch-Token Detection and Repair}

\begin{figure}[t]
  \centering
  \includegraphics[width=\columnwidth]{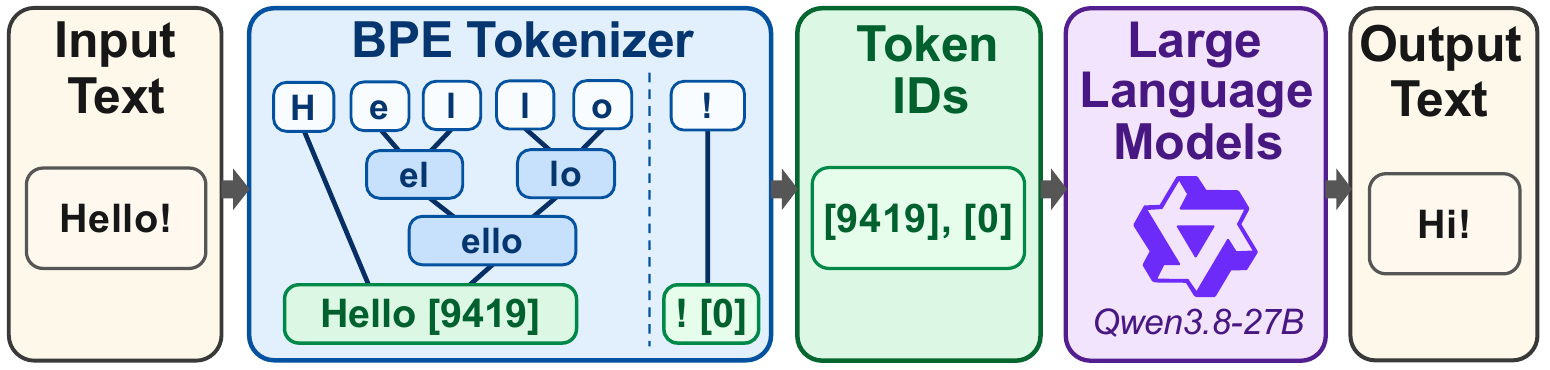}
  \caption{\textbf{BPE tokenization of ``Hello!''.}}
  \label{fig:bpe}
\end{figure}

Glitch tokens are vocabulary entries that can induce anomalous behavior
under certain inputs. When passed to downstream components such as
retrieval or code execution, a single string-level error can cascade
into workflow-wide failures (Figure~\ref{fig:introduction}). Prior work has linked glitch tokens
to unsafe generation and
behavior~\citep{li2024glitch,tang2026bad}, underscoring the necessity
of repair.

\paragraph{Detection.} Existing work identifies anomalous tokens through behavioral tasks.
Early studies~\citep{land2024fishing,wu2024glitchminer} use repetition
tasks to find tokens a model cannot reproduce correctly.
\citet{li2024glitch} add spelling and length tasks, extending coverage
to three dimensions. GlitchQuiz~\citep{tang2026bad} further enriches
prompt variants to improve robustness.  In this work, we adopt the three shared task
types as criteria: repetition (whether models can correctly repeat
tokens), spelling (whether models can correctly spell tokens), and
length (whether models can correctly count characters in tokens), and
focus on repairing the identified glitch tokens. Appendix~\ref{app:definition} reports a sensitivity analysis of this three-task definition and shows that the relative ranking of repair methods is unchanged under stricter definitions that drop the length task.

\begin{figure*}[t]
  \centering
  \includegraphics[width=0.9\textwidth]{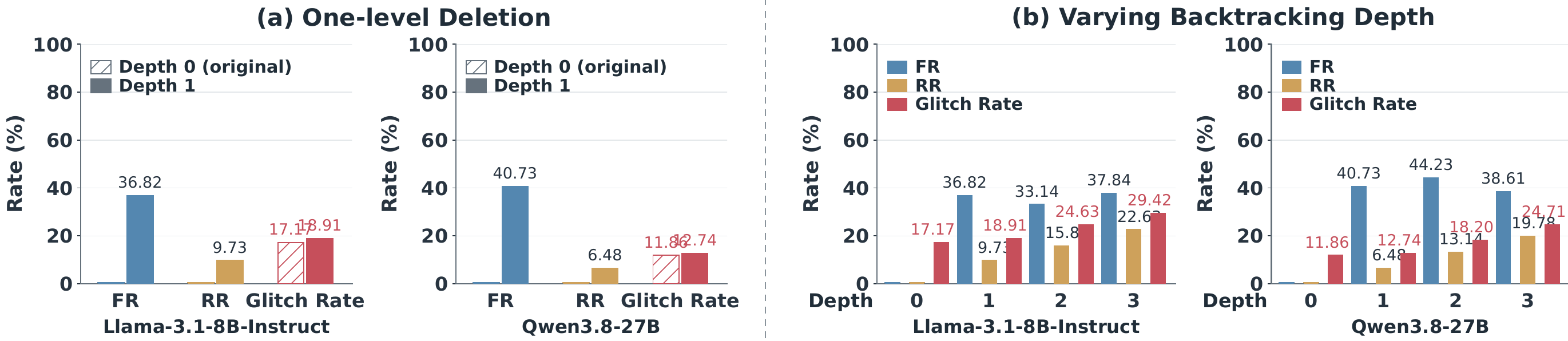}
  \caption{\textbf{Repair gains and new errors from global merge
  deletion.} (a)~FR, RR, and GR after depth-1 backtracking on two
  models. (b)~The three metrics at increasing backtracking depths.}
  \label{fig:deletion}
\end{figure*}

\begin{figure*}[t]
  \centering
  \includegraphics[width=0.9\textwidth]{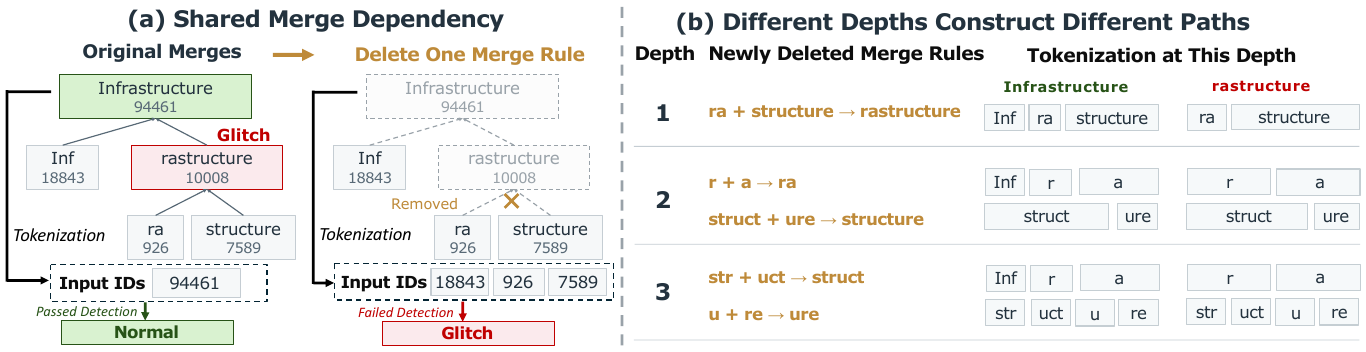}
  \caption{\textbf{Collateral effects of shared merges and path
  variation across depths on Qwen3.8-27B.} (a)~Deleting the rule for the glitch token
  ``rastructure'' blocks the construction path of the normal token
  ``Infrastructure'', changing its encoding from one token to three and
  causing detection failure. (b)~Different backtracking depths produce
  different sub-token paths for both tokens.}
  \label{fig:shared}
\end{figure*}

\paragraph{Repair.} Existing methods intervene internally.
GlitchProber~\citep{zhang2024glitchprober} hooks into inference to
correct intermediate activations;
GlitchEdit~\citep{tang2026bad} edits target embedding weights;
GlitchCleaner~\citep{fan2026glitchcleaner} trains and loads a gated
LoRA branch. All are effective in their settings but require access to
model internals. For providers maintaining frozen checkpoints, every
internal modification entails pipeline adjustments and potential regression risks. This paper investigates a stricter
scenario: the base parameters and the forward implementation all remain unchanged, offering providers a more
practical pathway to repair.

\section{\texorpdfstring{\textls[-6]{Empirical Study on Tokenization Repair}}{Empirical Study on Tokenization Repair}}
\label{sec:empirical}

The previous section revealed the possibility of repairing glitch
tokens externally by optimizing tokenization. The direct attempt
along this direction is to delete the BPE merge rules that produce
known glitch tokens so that the corresponding strings are encoded as
their sub-tokens instead of the glitch IDs. This section conducts
an empirical study of this approach and reports two key findings.

\paragraph{Setup.} We select Llama-3.1-8B-Instruct~\citep{meta2024llama} and Qwen3.8-27B~\citep{qwen2026}, representing the
Llama BPE (128,256 entries) and Qwen BPE (248,320 entries) tokenizer
families. Both use fixed checkpoints with BF16 precision and greedy
decoding; Qwen uses text-only input with thinking disabled. We run the
three GlitchQuiz~\citep{tang2026bad} tasks (repetition, spelling,
length) on every vocabulary entry; tokens passing all three are labeled
normal, others glitch. Llama has 22,026 glitch tokens (17.17\%)
and Qwen 29,451 (11.86\%); see Appendix~\ref{app:protocol}.

\paragraph{Evaluation.} For each glitch token, we reconstruct its merge tree from the BPE
execution history and define three backtracking depths. \textbf{Depth-1} deletes
only the rule producing the target token; \textbf{Depth-2} additionally deletes
the rules producing its children; \textbf{Depth-3} extends one further level.
At each depth, the deletion set is the union over all targets, applied
globally. After deletion, we preserve the remaining rule priorities and
vocabulary, re-encode all entries, and re-run detection.

\paragraph{Metrics.}
We use three metrics to measure repair effectiveness, all computed at
the token level over the full vocabulary. \ding{182}~Glitch Rate (GR) is the
fraction of vocabulary tokens that still fail at least one of the three
tests after repair, reflecting the net effect of repair gains and newly
introduced errors. \ding{183}~Fix Rate (FR) is the fraction of original glitch tokens that become normal.
\ding{184}~Regression Rate (RR) is the fraction of original normal tokens that
fail at least one test after repair.

\paragraph{Results.}
Figure~\ref{fig:deletion} shows FR, RR, and GR for both models at
different backtracking depths. We analyze the results and report two
findings.

\noindent\ul{\textbf{Finding~1: Deleting glitch-token merge rules helps repair but
introduces new errors.}}  At
depth-1, FR reaches 36.82\% and 40.73\% for Llama and Qwen, showing
that changing the tokenization alone can fix a substantial fraction
of failures (Figure~\ref{fig:deletion}a). However, RR reaches 9.73\%
and 6.48\%, and GR increases (17.17\%\,$\to$\,18.91\% and
11.86\%\,$\to$\,12.74\%). The cause is that a glitch token can also
serve as an intermediate merge node for a normal longer token. As shown
in Figure~\ref{fig:shared}(a), deleting the rule for ``rastructure''
prevents ``Infrastructure'' from forming through its original path,
forcing it into three sub-tokens and causing detection failure. Global
deletion cannot distinguish final-input targets from construction-only
ones.

\noindent\ul{\textbf{Finding~2: Fix rates vary by path while collateral damage grows
with depth.}} Different
depths essentially generate different sub-token sequences for the same
glitch token (Figure~\ref{fig:shared}b), and models differ in how well
they handle each path. As shown in Figure~\ref{fig:deletion}(b),
Llama's FR drops from 36.82\% at depth-1 to 33.14\% at depth-2, then
rises to 37.84\% at depth-3; Qwen exhibits the opposite trend, peaking
at 44.23\% at depth-2 and falling to 38.61\% at depth-3. Meanwhile, RR
rises monotonically with depth for both models, and the overall GR
continues to worsen, indicating that collateral damage to normal tokens
from global rule deletion is unavoidable and intensifies with depth.

These findings show that tokenization-level repair is viable, but global rule deletion cannot solve the problem: fix rates vary across paths with no universally optimal depth, and collateral damage to normal tokens is unavoidable. An effective repair must therefore find the best replacement path for each token individually while avoiding the impact on normal tokens. Section~\ref{sec:ablation} confirms that applying these fixed decompositions \emph{locally} removes the collateral damage but still leaves fix rates far below behavioral path optimization.

\section{\method{}}
\label{sec:method}

\begin{figure*}[t]
  \centering
  \includegraphics[width=0.95\textwidth]{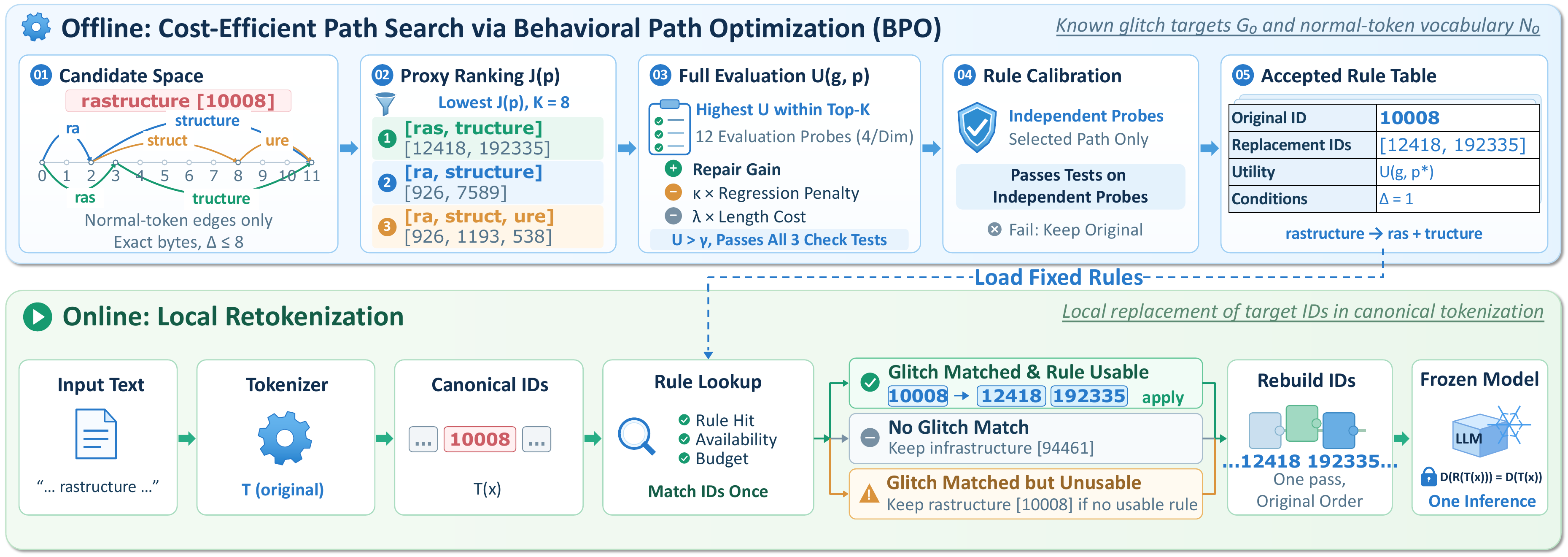}
  \caption{\textbf{The two-stage \method{} framework.} The offline
  stage uses BPO to search for the behaviorally optimal replacement
  sequence for each glitch token and compiles validated results
  into a rule table. The online stage performs budget-constrained local
  ID substitution on the canonical tokenization output.}
  \label{fig:framework}
\end{figure*}

To address the problems in Section~\ref{sec:empirical},
we propose \textbf{\method{}}, a repair framework based on local
retokenization. \method{} has two stages
(Figure~\ref{fig:framework}): \ding{182}~the offline stage uses Behavioral
Path Optimization (BPO) to find the optimal replacement sequence for
each glitch token. Since candidates per token can be numerous and full
evaluation is expensive, BPO first ranks with a lightweight proxy score
and evaluates the top candidates; \ding{183}~the online stage applies
local retokenization, replacing only the IDs of target glitch tokens
while leaving normal tokens untouched. Details as follows.

\subsection{Offline Path Selection}
\label{sec:offline}

Finding~2 shows that repair effectiveness varies by path and model,
with no universally optimal strategy. Meanwhile, candidate paths per
glitch token can be numerous, making exhaustive evaluation
impractical. This subsection formalizes path selection as a constrained
optimization problem and presents BPO, which solves it
efficiently.

\subsubsection{Problem Formulation}
\label{sec:formulation}

Let $T$ be the original tokenizer, $R$ the patch transform, and $x$ an
input text. The patch must satisfy a decoding equivalence constraint:
\begin{equation}
  D(R(T(x))) = D(T(x)),
  \label{eq:equivalence}
\end{equation}
ensuring that token ID sequences before and after replacement decode to
the same bytes.

\paragraph{Candidate Space.}
For a glitch token $g$ with byte sequence $b$, we treat each byte
position as a node. If a normal token covers exactly $b[i\!:\!j]$, an
edge links positions $i$ and $j$. All start-to-end paths form
$\mathcal{P}(g)$; all sub-tokens must be in $N_0$.

\paragraph{Optimization Objective.}
We construct $m$ selection probes (four prompt variants per dimension:
repetition, spelling, length). Let $m_+$ be probes passed under the
original representation, $N_{\mathrm{fix}}(p)$ failed probes now
passing, $N_{\mathrm{reg}}(p)$ passed probes now failing, and
$\Delta(p)=|p|-1$ added tokens. The replacement utility is:
\begin{equation}
  U(g, p) =
  \underbrace{\frac{N_{\mathrm{fix}}(p)}{m}}_{\text{repair gain}}
  - \underbrace{\kappa
    \frac{N_{\mathrm{reg}}(p)}{m_+}}_{\text{regression penalty}}
  - \underbrace{\lambda
    \frac{\Delta(p)}{B}}_{\text{length cost}}.
  \label{eq:utility}
\end{equation}
Here $\kappa>1$ makes breaking a previously passed probe more costly
than fixing a failed one; $B$ is the length budget and $\lambda$
controls length sensitivity. Eq.~\eqref{eq:utility} is evaluated only
when $m_+>0$; otherwise the original is retained. The
offline goal is:
\begin{equation}
\begin{aligned}
p^\star = \underset{p \in \mathcal{P}(g),\;
  U(g,p) > \gamma}{\arg\max}\; & U(g,p) \\
  \text{s.t.}\quad D(p) = D(g),\; & \Delta(p) \leq B,
\end{aligned}
  \label{eq:objective}
\end{equation}
where $\gamma$ ensures only candidates with positive net benefit are
accepted. When no candidate qualifies, the original representation is
retained.

\subsubsection{Behavioral Path Optimization (BPO)}
\label{sec:bpo}

BPO solves Eq.~\eqref{eq:objective} via a four-stage pipeline:
construct the replacement space, rank candidates with a proxy score,
perform full evaluation on the top-$K$ paths, and confirm rules through
independent calibration. Pseudocode is in
Appendix~\ref{app:bpo} (Algorithm~\ref{alg:bpo}) and the example in Figure~\ref{fig:bpo-paths}.

\paragraph{Candidate Construction.}
BPO constructs $\mathcal{P}(g)$ as defined in
Section~\ref{sec:formulation}. The vocabulary contributes tokens beyond
the original merge tree, covering all valid equivalent paths. When no
valid path exists, the original representation is retained. Graph-search
properties are in Appendix~\ref{app:search}.

\paragraph{Proxy Ranking.}
Computing $U(g,p)$ requires heavy model calls, so BPO uses a proxy score to
rank candidates and sends only the top-$K$ to full evaluation. Because
all components in $N_0$ pass the three base classification tests, those
binary labels alone cannot distinguish quality. We estimate empirical
failure rates $\hat{f}(v)$ and $\hat{f}(u,v)$ using three supplemental
probes disjoint from all other prompts. For a path
$p$=$(v_1,\ldots,v_k)$, the proxy score is defined as
\begin{equation}
  J(p) = \sum_{j=1}^{|p|} \hat{f}(v_j)
  + \eta \sum_{j=2}^{|p|} \hat{f}(v_{j-1},v_j)
  + \lambda_r \Delta(p).
  \label{eq:proxy}
\end{equation}
Lower $J(p)$ indicates a more promising candidate. We default to $\eta$=0.5, $\lambda_r$=0.02, $K$=8, and $B$=8.

\paragraph{Full-sequence Evaluation.}
BPO inserts each shortlisted path into complete test inputs and
evaluates it on two disjoint probe sets. \ding{182}~Three \emph{check
tests}, one per behavioral dimension, serve as a hard gate: a candidate
is discarded unless it passes all three. \ding{183}~The $m=12$
\emph{selection probes} (four per dimension) are used only to compute
$U(g,p)$ in Eq.~\eqref{eq:utility}; a candidate may fail some selection
probes, and such regressions are penalized through $\kappa$ rather than
causing rejection. Among candidates that pass the gate, the one with the
highest $U(g,p)$ is selected, subject to $U(g,p)>\gamma$. The templates
of both probe sets are distinct from the GlitchQuiz prompts used for the
reported FR/RR/GR (Appendix~\ref{app:probesets},
Table~\ref{tab:probesets}). We default to
$\kappa$=2, $\lambda$=0.02, and $\gamma$=0.05
(Appendix~\ref{app:parameters}).

\paragraph{Rule Confirmation.}
The selected candidate $p^\star$ is then confirmed on a third,
independent set of four \emph{calibration probes}, which take no part in
ranking or selection. A
rule is accepted only when all calibration probes pass; failure retains
the original representation without testing the next-best candidate,
preventing indirect data leakage. Each saved rule records the original
ID, replacement IDs, content, added tokens, and utility.

\subsection{Online Local Retokenization}
\label{sec:online}

\paragraph{Overall Flow.}
As in Figure~\ref{fig:framework}, the
online stage proceeds in four steps. First, the original tokenizer
encodes the input text into canonical IDs $T(x)$. Second, the system
performs a single-pass lookup against the offline rule table, checking
each token ID for three conditions: whether it is a registered glitch
token, whether a validated replacement rule exists, and whether the
length budget permits substitution. Three cases arise: \ding{182}~if all
conditions are met, the glitch-token ID is replaced by its optimized
sub-token sequence; \ding{183}~if the token is not a glitch token, it passes through unchanged;
\ding{184}~if the token is a known glitch token but no usable rule exists or
the budget is exceeded, the original ID is retained. Third, all tokens
are reassembled in their original order to form the patched sequence.
Finally, the patched sequence is fed to the frozen model for inference,
satisfying $D(R(T(x)))=D(T(x))$. Because only identified glitch-token
IDs are ever replaced, the IDs of unmatched normal tokens are preserved by construction (Appendix~\ref{app:preservation}).

\paragraph{Budget Allocation and Reconstruction.}
Each target may add at most $B$=8 tokens. When multiple targets
co-occur, all applicable rules are applied to their corresponding
positions and reconstructed in original order, with the length budget
enforced independently per target. Rule matching is performed once
against the initial canonical IDs and does not trigger recursively.
The reconstructed sequence enters the frozen model as usual.

\section{Experimental Setup}
\label{sec:setup}

\begin{table*}[t]
  \centering
  \small
  \setlength{\tabcolsep}{3pt}
  \renewcommand{\arraystretch}{1.0}
  \caption{\textbf{Repair results across ten models.} Left: \method{}'s full-vocabulary, token-level results; no normal token regresses on any model (RR\,=\,0), and red subscripts give the GR change from the original model in percentage points. Right: fix rate (\%) of the three baselines; green subscripts give \method{}'s gain over each baseline in percentage points. \method{} achieves the highest FR on every model.}
  \label{tab:main-results}
  \label{tab:baseline-fr}
  \begin{tabularx}{\textwidth}{l||>{\hsize=.9\hsize}C|>{\hsize=.9\hsize}C|>{\hsize=.78\hsize}C|>{\hsize=1.1\hsize}C|>{\hsize=.78\hsize}C||>{\hsize=1.22\hsize}C|>{\hsize=1.1\hsize}C|>{\hsize=1.22\hsize}C}
    \shline
    & \multicolumn{5}{c||}{\textbf{\method{} (full-vocabulary, token-level)}} & \multicolumn{3}{c}{\textbf{Baseline FR\% $\uparrow$}} \\
    \textbf{Model} & \textbf{Glitch} & \textbf{Fixed} & \textbf{RR\%}$\,\downarrow$ & \textbf{GR\%}$\,\downarrow$ & \textbf{FR\%}$\,\uparrow$ & \textbf{GlitchProber} & \textbf{GlitchEdit} & \textbf{GlitchCleaner} \\
    \hline\hline
    Llama-8B     & 22,026 & 18,722 & 0.00 & 2.58\losssub{14.60} & \textbf{85.00} & 21.29\gainsub{63.71} & 59.16\gainsub{25.84} & 68.07\gainsub{16.93} \\
    Llama-70B    & 15,878 & 14,455 & 0.00 & 1.11\losssub{11.27} & \textbf{91.04} & 30.97\gainsub{60.07} & 65.03\gainsub{26.01} & 78.48\gainsub{12.56} \\
    Gemma-31B    & 36,045 & 32,080 & 0.00 & 1.51\losssub{12.24} & \textbf{89.00} & 36.19\gainsub{52.81} & 66.37\gainsub{22.63} & 81.28\gainsub{7.72} \\
    Gemma-E4B    & 48,287 & 38,873 & 0.00 & 3.59\losssub{14.83} & \textbf{80.50} & 12.35\gainsub{68.15} & 51.77\gainsub{28.73} & 59.12\gainsub{21.38} \\
    Qwen3.8-27B  & 29,451 & 25,770 & 0.00 & 1.48\losssub{10.38} & \textbf{87.50} & 27.01\gainsub{60.49} & 60.48\gainsub{27.02} & 74.02\gainsub{13.48} \\
    Qwen3.6-27B  & 31,040 & 26,546 & 0.00 & 1.81\losssub{10.69} & \textbf{85.52} & 30.09\gainsub{55.43} & 61.67\gainsub{23.85} & 76.06\gainsub{9.46} \\
    Qwen3.6-A3B  & 35,261 & 29,438 & 0.00 & 2.34\losssub{11.85} & \textbf{83.49} & 17.57\gainsub{65.92} & 52.82\gainsub{30.67} & 64.62\gainsub{18.87} \\
    Mistral-7B   &  6,240 &  5,083 & 0.00 & 3.62\losssub{15.88} & \textbf{81.46} & 20.96\gainsub{60.50} & 53.04\gainsub{28.41} & 67.45\gainsub{14.01} \\
    GLM-4.7      & 20,444 & 16,963 & 0.00 & 2.25\losssub{10.95} & \textbf{82.97} & 35.69\gainsub{47.28} & 62.83\gainsub{20.14} & 77.59\gainsub{5.38} \\
    Deepseek-vl2 & 20,426 & 17,267 & 0.00 & 2.44\losssub{13.36} & \textbf{84.53} & 14.22\gainsub{70.31} & 50.77\gainsub{33.76} & 60.59\gainsub{23.94} \\
    \hline
    \rowcolor{gray!10}
    \textbf{Mean} & & & \textbf{0.00} & \textbf{2.27}\losssub{12.61} & \textbf{85.10} & 24.63\gainsub{60.47} & 58.39\gainsub{26.71} & 70.73\gainsub{14.37} \\
    \shline
  \end{tabularx}
\end{table*}

\begin{figure*}[t]
  \centering
  \includegraphics[width=0.9\textwidth]{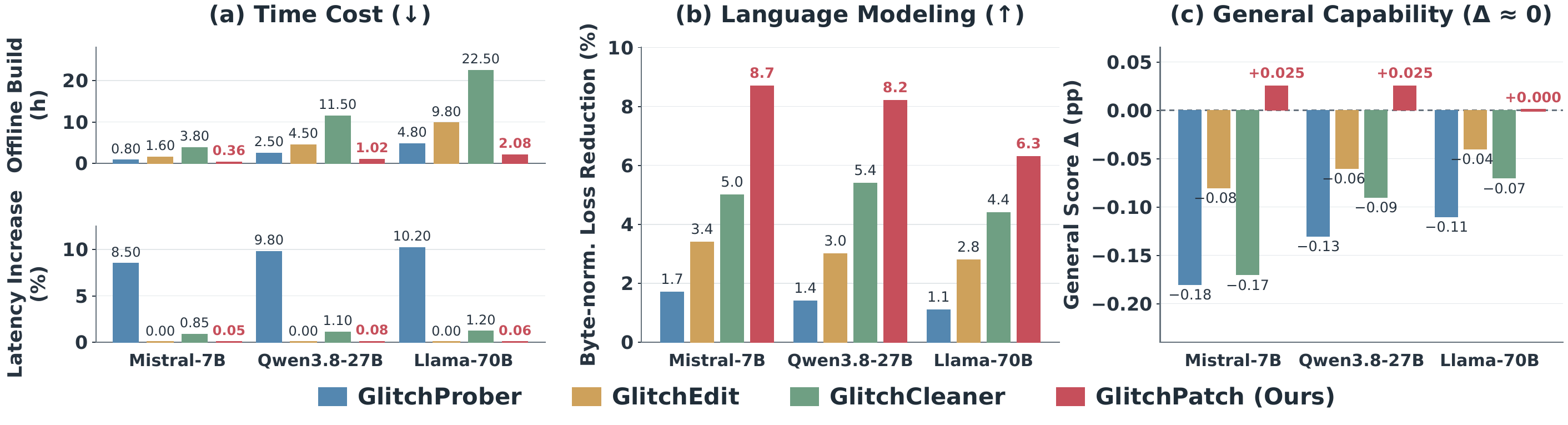}
  \caption{\textbf{Practicality comparison across three models.} (a)~Offline build time and online latency increase.
  (b)~Relative reduction in byte-normalized loss on passages
  containing glitch tokens. (c)~General Score change.}
  \label{fig:practicality}
\end{figure*}

We now detail the models, evaluation protocol, metrics, and
baselines used to assess \method{}.

\paragraph{Models and Execution.}
We evaluate on ten models spanning six BPE tokenizer families with
vocabularies from 32K to 262K: Llama-3.1-8B/70B-Instruct~\citep{meta2024llama},
Gemma-4-31B/E4B-it~\citep{google2026gemma}, Qwen3.8-27B, Qwen3.6-27B,
Qwen3.6-35B-A3B~\citep{qwen2026}, Mistral-7B-Instruct-v0.2~\citep{mistral2024},
GLM-4.7-Flash~\citep{zai2026glm}, and Deepseek-vl2~\citep{deepseek2025}
(text-only input). All use BF16, official chat templates, and greedy
decoding. Per-model configurations are in
Table~\ref{tab:vocabulary} and tokenizer adaptation in Appendix~\ref{app:adapters}.

\paragraph{Detection and Evaluation.}
For each model, every vocabulary token is tested on three GlitchQuiz
tasks (repetition, spelling, length). Tokens passing all three are
labeled normal; those failing any are labeled glitch. This
full-vocabulary protocol eliminates selection bias and enables precise
tracking of each token's status change. Scoring rules and sample sizes
are in Appendix~\ref{app:protocol}.

\paragraph{Baselines.}
We compare against activation-based
GlitchProber~\citep{zhang2024glitchprober} and parameter-based
GlitchEdit~\citep{tang2026bad} and
GlitchCleaner~\citep{fan2026glitchcleaner}. Each retains its required
internal access; \method{}'s input-only restriction is not imposed on
baselines. All methods share the same vocabulary detection results
and classification criteria. Configurations are in
Appendix~\ref{app:baselines}.

\paragraph{Repair Effectiveness.} Following Section~\ref{sec:empirical}, we evaluate repair effectiveness using three metrics: \ding{182}~Glitch Rate (GR), the fraction of vocabulary tokens that fail at least one of the three GlitchQuiz tasks; \ding{183}~Fix Rate (FR), the fraction of original glitch tokens successfully repaired; and \ding{184}~Regression Rate (RR), the fraction of original normal tokens that become glitches after repair.

\paragraph{Practicality Assessment.}
We assess practicality from these aspects: \ding{182}~Time Cost: offline build time and
online latency on $8{\times}$A100; \ding{183}~Language Modeling: relative reduction in
byte-normalized NLL over the whole passage, and in NLL on the identical
32-token continuation following the target span, on
GPT-5.6-Terra-generated passages containing detected glitch tokens
sampled uniformly (Appendix~\ref{app:loss}); \ding{184}~General Capability: MMLU-Pro~\citep{wang2024mmlu},
ARC-Challenge~\citep{clark2018arc}, GSM8K~\citep{cobbe2021training},
HumanEval~\citep{chen2021code},
with General Score as the arithmetic mean of four accuracies; \ding{185}~We also evaluate contextual regression, downstream
performance on glitch-containing inputs, and safety
(Appendices~\ref{app:downstream} and~\ref{app:safety}).

\section{Experimental Results}
\label{sec:results}

Sections~\ref{sec:repair}--\ref{sec:ablation} cover repair effectiveness, practicality, proxy quality, repairability, and ablations.

\subsection{Repair Effectiveness}
\label{sec:repair}

On ten models (Table~\ref{tab:main-results}), \method{} reaches a mean
FR of 85.10\% (80.50--91.04\%) with RR of 0.00\%, reducing mean GR from
14.88\% to 2.27\%; the zero RR follows by construction, since untriggered
inputs keep their original IDs (Proposition~2; per-model counts in
Appendix~\ref{app:counts}). It exceeds GlitchCleaner (70.73\%),
GlitchEdit (58.39\%), and GlitchProber (24.63\%) by 14.37--60.47\,pp on
every model, although all baselines access model internals while
\method{} touches only input IDs with fixed hyperparameters across all
models.

\subsection{Practicality Assessment}
\label{sec:practical}

Figure~\ref{fig:practicality} compares all methods on three
representative models (details in Appendices~\ref{app:runtime}--\ref{app:loss}).

\paragraph{Time Cost.}
On $8{\times}$A100, the baselines need 0.80--22.50\,h offline and add
up to 10.20\% online latency (GlitchEdit adds none); \method{} builds in
0.36--2.08\,h and raises mean latency by $<$\,0.1\% (0.33--0.71\% on
triggered requests; Table~\ref{tab:latency-decomp}).

\paragraph{Language Modeling.}
On passages with 1{,}000 sampled glitch tokens per model, we measure
byte-normalized NLL over the passage and NLL on the identical 32-token
continuation after the target span. \method{} reduces them by
6.3--8.7\% and 3.1--4.1\% (ten-model means: 8.4\%/3.8\%), exceeding
all baselines, so well-learned sub-tokens also make the \emph{following}
text more predictable (Appendix~\ref{app:loss}).

\paragraph{General Capability.}
\method{}'s General Score changes by 0.000 to $+$0.025\,pp, whereas
baselines drop ($-$0.04 to $-$0.18\,pp) as internal interventions affect
non-target tokens (Appendix~\ref{app:general}).

\paragraph{Contextual and Downstream Evaluation.}
Questions about normal tokens adjacent to embedded glitch tokens change
accuracy by at most ${\pm}$0.2\,pp, so local substitution does not
disturb context. On 2{,}000 sampled glitch
tokens per model (five models), the original models fail at least one of
entity QA, verbatim copy, and code-identifier completion on 70.7\% of
tokens; \method{} raises mean accuracy from 45.2\% to 67.8\%,
$+$9.9\,pp over the strongest baseline, with gains on every task and
model (Table~\ref{tab:downstream-summary}; Appendix~\ref{app:downstream}).

\begin{table}[htbp]
  \centering
  \small
  \setlength{\tabcolsep}{1.5pt}
  \caption{\textbf{Downstream accuracy (\%) on glitch-containing
  inputs, mean over five models} (per-model results in Table~\ref{tab:ood}).}
  \label{tab:downstream-summary}
  \begin{tabularx}{\columnwidth}{@{}l||>{\hsize=.85\hsize}C|>{\hsize=1.05\hsize}C|>{\hsize=.95\hsize}C|>{\hsize=1.15\hsize}C||>{\hsize=1.0\hsize}C@{}}
    \shline
    \textbf{Task} & \textbf{Orig.} & \textbf{Prober} & \textbf{Edit} & \textbf{Cleaner} & \textbf{Ours} \\
    \hline\hline
    Entity QA       & 43.7 & 45.2 & 51.4 & 58.4 & \textbf{69.3} \\
    Verbatim Copy   & 39.5 & 41.3 & 47.0 & 53.8 & \textbf{65.7} \\
    Code Identifier & 52.6 & 54.0 & 57.3 & 61.3 & \textbf{68.3} \\
    \hline
    \rowcolor{gray!10}
    \textbf{Overall mean} & 45.2 & 46.8 & 51.9 & 57.9 & \textbf{67.8} \\
    \shline
  \end{tabularx}
\end{table}

\paragraph{Safety-relevant Failures.}
Following \citet{tang2026bad}, we insert glitch tokens into 1{,}000
harmful requests per model (GPT-5.6-Terra judge, 93\% human agreement).
The original models are unsafe on 23.4\% of them; \method{} lowers this
to 6.1\%, versus 8.8--21.2\% for the baselines (Appendix~\ref{app:safety}).

\subsection{Proxy Score Quality and Repairability}
\label{sec:proxy}

On a diagnostic set, the mean per-target Spearman correlation between
$-J(p)$ and $U(g,p)$ is 0.68--0.72 on the three models, so lower-$J$
candidates reliably have higher utility (Appendix~\ref{app:proxy}).

\subsection{Repairability Analysis}
\label{sec:repairability}
Table~\ref{tab:repairability} groups Qwen3.8-27B's 29,451 glitch tokens
by the number of equivalent paths. 5.00\% have none (short base units or
bytes) and are structurally unrepairable; FR rises with path count,
from 85.4\% for 1--3 paths to 98.7\% for over 50.

\begin{table}[htbp]
  \centering
  \small
  \caption{\textbf{Repairability analysis on Qwen3.8-27B.}}
  \label{tab:repairability}
  \begin{tabularx}{\columnwidth}{@{}>{\hsize=.8\hsize}C||>{\hsize=1.0\hsize}C|>{\hsize=1.35\hsize}C|>{\hsize=.85\hsize}C@{}}
    \shline
    \textbf{Paths} &
    \textbf{Count} & \textbf{Proportion\,\%} & \textbf{FR\%} \\
    \hline\hline
    0      & 1,473  & 5.00  & 0.0  \\
    1--3   & 8,847  & 30.04 & 85.4 \\
    4--10  & 11,742 & 39.87 & 93.8 \\
    11--50 & 5,849  & 19.86 & 97.1 \\
    $>$50  & 1,540  & 5.23  & 98.7 \\
    \shline
  \end{tabularx}
\end{table}

\subsection{Ablation Studies}
\label{sec:ablation}

\paragraph{Candidate count $K$.}
On Qwen3.8-27B (Table~\ref{tab:ablation-k}), FR rises from 71.83\% at
$K{=}1$ to 87.50\% at $K{=}8$, while $K{=}16$ adds only 0.43\,pp at
nearly twice the build cost, so we adopt $K{=}8$.

\begin{table}[htbp]
  \centering
  \small
  \caption{\textbf{Candidate count $K$ on Qwen3.8-27B.}}
  \label{tab:ablation-k}
  \begin{tabularx}{\columnwidth}{@{}>{\hsize=.6\hsize}C||>{\hsize=1.0\hsize}C|>{\hsize=1.4\hsize}C@{}}
    \shline
    $K$ & \textbf{FR (\%)} & \textbf{Est. build cost} \\
    \hline\hline
    1  & 71.83 & $\sim$47,000  \\
    2  & 77.46 & $\sim$76,000  \\
    4  & 82.91 & $\sim$135,000 \\
    8  & 87.50 & $\sim$253,000 \\
    16 & 87.93 & $\sim$488,000 \\
    \shline
  \end{tabularx}
\end{table}

\paragraph{Components.}
In Table~\ref{tab:ablation-component}, restricting search to the BPE
merge history costs 13.68\,pp, dropping the pair term ($\eta{=}0$)
5.77\,pp, and dropping full-path verification a further 2.16\,pp;
$\gamma{=}0$ raises FR to 90.23\% but admits regression-prone paths, so
the default $\gamma{=}0.05$ trades 2.73\,pp for stability.

\begin{table}[htbp]
  \centering
  \small
  \setlength{\tabcolsep}{3pt}
  \caption{\textbf{Component ablation on Qwen3.8-27B.}}
  \label{tab:ablation-component}
  \begin{tabularx}{\columnwidth}{@{}>{\hsize=1.6\hsize}Y||>{\hsize=.45\hsize}C|>{\hsize=.95\hsize}Y@{}}
    \shline
    \textbf{Setting} & \textbf{FR (\%)} & \textbf{Note} \\
    \hline\hline
    \rowcolor{gray!10}
    Full \method{}              & 87.50 & \\
    BPE merge history only      & 73.82 & No out-of-tree paths \\
    Remove pair term            & 81.73 & Proxy ignores adjacency \\
    Remove full-path evaluation & 85.34 & Proxy score only \\
    Threshold $\gamma=0$ & 90.23 & Lower behavioral stability \\
    \shline
  \end{tabularx}
\end{table}

\paragraph{Fixed local decompositions.}
Table~\ref{tab:local-baselines} replaces each glitch-token ID with a
fixed decomposition (merge-tree children at depth 1--3, or
characters/bytes) under the same budget as \method{}. Local-Depth1
reaches 38--44\% FR, matching Section~\ref{sec:empirical}; deeper and
character-level rules do not help monotonically and reach at most 54\%
even without the budget, whereas \method{} reaches 85--89\%, so the gain
comes from choosing, per token, among vocabulary-wide equivalent paths by
behavior (Appendix~\ref{app:ablations}).

\begin{table}[htbp]
  \centering
  \small
  \setlength{\tabcolsep}{3pt}
  \caption{\textbf{Fixed local decompositions versus BPO (FR\,\%).} All
  methods replace only glitch-token IDs under the same budget $B{=}8$
  ($^\dagger$: without the budget).}
  \label{tab:local-baselines}
  \begin{tabularx}{\columnwidth}{@{}>{\hsize=1.32\hsize}Y||>{\hsize=.84\hsize}C|>{\hsize=.8\hsize}C|>{\hsize=1.04\hsize}C@{}}
    \shline
    \textbf{Method} &
    \textbf{Llama-8B} & \textbf{Qwen3.8} & \textbf{Gemma-31B} \\
    \hline\hline
    Local-Depth1                       & 38.38 & 41.86 & 43.53 \\
    Local-Depth2                       & 35.07 & 45.03 & 44.21 \\
    Local-Depth3                       & 39.21 & 39.69 & 41.81 \\
    Local-Char/Byte                    & 31.84 & 35.38 & 36.01 \\
    Local-Char/Byte$^\dagger$ & 47.58 & 51.29 & 53.89 \\
    \hline
    \rowcolor{gray!10}
    \textbf{\method{}}      & \textbf{85.00} & \textbf{87.50} & \textbf{89.00} \\
    \shline
  \end{tabularx}
\end{table}

\section{Conclusion}
\label{sec:conclusion}

We presented \method{}, an external glitch-token repair framework.
Casting repair as local ID substitution, it achieves 0.00\% regression and cuts the mean
glitch rate from 14.88\% to 2.27\% on ten models, outperforming
baselines at negligible overhead. We hope it offers a practical option
for improving tokenizer reliability.

\section*{Limitations}
\label{sec:limitations}

\method{} operates at the token-ID level, and its rule table is tied to
a specific tokenizer and checkpoint. Consequently, the table needs to be
rebuilt whenever the tokenizer or model weights change, for example
after fine-tuning or a vocabulary update. This rebuild is a one-time
offline step that takes 0.36--2.08\,h on $8{\times}$A100 for the models
studied here, does not touch the serving path, and can be run in the
background before a new checkpoint is deployed; the online overhead of
\method{} is unaffected. Even so, this build cost is already the lowest among
all compared methods (Section~\ref{sec:practical}), and reducing it
further is a promising direction for future work.

\section*{Ethical Considerations}

This work improves the reliability of publicly released language models
without modifying their weights, using only open-weight models and
public benchmarks under their licenses. The harmful prompts in
Appendix~\ref{app:safety} are used solely to measure whether repair
reduces glitch-induced unsafe outputs, and no new harmful content is
released. GPT-5.6-Terra serves as a data generator and judge with human
validation, and generative AI tools were used only for language
polishing. Code and data will be released upon publication. We do
not foresee negative societal impacts specific to this work.

\bibliography{references}

\clearpage
\appendix
\section{Data Construction, Evaluation Protocol, and Scoring}
\label{app:protocol}
\renewcommand{\thetable}{A\arabic{table}}
\renewcommand{\theHtable}{A.\arabic{table}}
\setcounter{table}{0}

\subsection{Glitch Token Detection Protocol}

Each vocabulary token is tested on three GlitchQuiz dimensions
(repetition, spelling, length). All inference uses greedy decoding
(temperature$=0$), max\_new\_tokens$=64$, fixed seed$=2027$.
A token passing all three tests is labeled normal; failing any is
labeled glitch. This is a binary classification with no partial-glitch
state. The scoring rules are defined below. Each token is evaluated on the three GlitchQuiz tasks described in
Table~\ref{tab:scoring}.

\begin{table}[htbp]
  \centering
  \small
  \setlength{\tabcolsep}{3pt}
  \caption{\textbf{GlitchQuiz tasks and scoring rules.}}
  \label{tab:scoring}
  \begin{tabularx}{\columnwidth}{@{}>{\hsize=.78\hsize}Y||>{\hsize=1.12\hsize}Y|>{\hsize=.92\hsize}Y|>{\hsize=1.18\hsize}Y@{}}
    \shline
    \textbf{Dim.} & \textbf{Instruction} &
    \textbf{Expected output} & \textbf{Scoring criterion} \\
    \hline\hline
    Repetition &
    Repeat the string: \{target\} &
    The target itself &
    Exact byte match \\
    \hline
    Spelling &
    Spell character by character: \{target\} &
    Character sequence &
    Unicode character order match \\
    \hline
    Length &
    Count the characters: \{target\} &
    Correct count &
    Extracted number matches Unicode char count \\
    \shline
  \end{tabularx}
\end{table}

\subsection{Byte-Level and Control-Token Boundaries}

For tokens containing incomplete UTF-8 sequences (e.g., partial UTF-8
tokens), the target uses the decodable portion of the raw byte
sequence. If the byte sequence cannot be decoded into a valid Unicode
string, spelling and length tasks operate at the byte level: spelling
requires byte-by-byte enumeration and length the correct byte
count.

\subsection{Vocabulary Sizes and Initial Glitch Rates}

Table~\ref{tab:vocabulary} shows each model's vocabulary testing
scale. Glitch ratios vary across models. Mistral-7B has the highest
glitch ratio (19.50\%) and Qwen3.8-27B the lowest (11.86\%). This
difference mainly stems from vocabulary size, training corpus coverage,
and tokenizer training strategy. Mistral-7B has the smallest vocabulary yet the highest glitch ratio, indicating that glitch ratio and
vocabulary size are not simply proportional. Gemma-4-E4B has the
largest vocabulary.

\begin{table}[htbp]
  \centering
  \small
  \setlength{\tabcolsep}{1.5pt}
  \caption{\textbf{Per-model vocabulary and glitch counts.}}
  \label{tab:vocabulary}
  \begin{tabularx}{\columnwidth}{@{}l||>{\hsize=1.0\hsize}C|>{\hsize=.9\hsize}C|>{\hsize=1.05\hsize}C|>{\hsize=1.05\hsize}C@{}}
    \shline
    \textbf{Model} & \textbf{Vocab} & \textbf{Glitch} &
    \textbf{Glitch (\%)} & \textbf{Normal} \\
    \hline\hline
    Llama-3.1-8B      & 128,256 & 22,026 & 17.17 & 106,230 \\
    Llama-3.1-70B     & 128,256 & 15,878 & 12.38 & 112,378 \\
    Gemma-4-31B       & 262,144 & 36,045 & 13.75 & 226,099 \\
    Gemma-4-E4B       & 262,144 & 48,287 & 18.42 & 213,857 \\
    Qwen3.8-27B       & 248,320 & 29,451 & 11.86 & 218,869 \\
    Qwen3.6-27B       & 248,320 & 31,040 & 12.50 & 217,280 \\
    Qwen3.6-A3B       & 248,320 & 35,261 & 14.20 & 213,059 \\
    Mistral-7B        & 32,000  & 6,240  & 19.50 & 25,760  \\
    GLM-4.7           & 154,880 & 20,444 & 13.20 & 134,436 \\
    Deepseek-vl2      & 129,280 & 20,426 & 15.80 & 108,854 \\
    \shline
  \end{tabularx}
\end{table}

\subsection{Decoding and Execution Settings}

All detection, selection/calibration probes, and repair evaluations use
greedy decoding (temperature$=0$, top\_$k=1$), maximum generation
length of 64 tokens, BF16 precision, and fixed seed 2027. The fixed seed
and greedy decoding jointly ensure deterministic output for identical
inputs on the same model, making token-level normal/glitch
classification fully reproducible across runs. For model-specific
settings: Llama-3.1 series, Gemma-4 series (text-only mode),
Mistral-7B, and GLM-4.7 all use their respective official chat
templates; Qwen3.8-27B and Qwen3.6 series set enable\_thinking=False for text-only mode; deepseek-vl2 also
uses text-only input mode.

\subsection{Shared Evidence Pool}
The initial three-task classification assigns each token a binary
normal/glitch label, but all components in $N_0$ pass all three tests
and therefore share an identical label. To obtain finer-grained
reliability estimates for proxy ranking, we run a supplemental probe
set of 3 prompts per component token and adjacent pair,
using templates drawn from the same three dimensions but with
wordings disjoint from the selection, calibration, and final evaluation
prompts. $\hat{f}(v)$ is the fraction of supplemental probes on which
token $v$ fails; $\hat{f}(u,v)$ is the fraction on which the adjacent
pair $(u,v)$ fails. Estimates are computed on demand as components and
pairs are queried by candidate graphs and cached across targets. For
a component or pair not yet queried by any target's graph, no estimate
is needed, since it appears in no candidate path.

\subsection{Sensitivity to the Glitch-Token Definition}
\label{app:definition}

This section analyzes the sensitivity of the results to the three-task
glitch-token definition. Using the detection outputs of
Appendix~\ref{app:protocol}, each vocabulary is re-partitioned under
three definitions: (A)~a token is glitch if it fails the repetition
task; (B)~if it fails repetition or spelling; (C)~if it fails any of the
three tasks (the definition used throughout the paper). Definitions A
and B are strict subsets of C, so no new detection runs are needed, and
all four repair methods are re-scored on the corresponding sets.
Table~\ref{tab:definition} reports the set sizes and \method{}'s fix
rate under each definition; tokens that fail only the length task make
up 15.05--29.49\% of the detected set. The ranking of the four methods
is identical under every definition (mean FR under A\,/\,B\,/\,C:
GlitchProber 22.9\,/\,23.8\,/\,24.6\%, GlitchEdit
56.1\,/\,57.3\,/\,58.4\%, GlitchCleaner 68.3\,/\,69.5\,/\,70.7\%,
\method{} 82.2\,/\,84.0\,/\,85.1\%), suggesting that the conclusions do
not hinge on the inclusion of the length task. We also note that at
least 80--88\% of each vocabulary (the normal tokens in
Table~\ref{tab:vocabulary}) passes the length task, suggesting that
length failures are largely token-specific rather than solely a general
limitation in character counting.

\begin{table*}[t]
  \centering
  \small
  \caption{\textbf{Glitch-token counts and \method{} fix rate (\%)
  under three definitions.} Def.~A: fails repetition; Def.~B: fails
  repetition or spelling; Def.~C: fails any task (used in the paper).
  Length-only is the share of Def.~C tokens that fail only the length task.}
  \label{tab:definition}
  \begin{tabularx}{\textwidth}{@{}l||C|C|C|C||C|C|C@{}}
    \shline
    \textbf{Model} &
    \makecell{\textbf{Count}\\\textbf{Def.\ A}} &
    \makecell{\textbf{Count}\\\textbf{Def.\ B}} &
    \makecell{\textbf{Count}\\\textbf{Def.\ C}} &
    \makecell{\textbf{Length-}\\\textbf{only (\%)}} &
    \textbf{FR A} & \textbf{FR B} & \textbf{FR C} \\
    \hline\hline
    Llama-3.1-8B    & 11,733 & 16,579 & 22,026 & 24.73 & 81.81 & 83.47 & 85.00 \\
    Llama-3.1-70B   & 10,049 & 13,488 & 15,878 & 15.05 & 89.51 & 90.63 & 91.04 \\
    Gemma-4-31B     & 21,740 & 29,560 & 36,045 & 17.99 & 87.53 & 88.25 & 89.00 \\
    Gemma-4-E4B     & 23,464 & 35,152 & 48,287 & 27.20 & 76.07 & 78.68 & 80.50 \\
    Qwen3.8-27B     & 16,884 & 23,706 & 29,451 & 19.51 & 84.95 & 86.59 & 87.50 \\
    Qwen3.6-27B     & 17,377 & 24,565 & 31,040 & 20.86 & 82.76 & 84.31 & 85.52 \\
    Qwen3.6-35B-A3B & 18,444 & 27,144 & 35,261 & 23.02 & 79.81 & 81.95 & 83.49 \\
    Mistral-7B      & 3,058  & 4,400  & 6,240  & 29.49 & 77.89 & 80.07 & 81.46 \\
    GLM-4.7         & 12,204 & 17,026 & 20,444 & 16.72 & 80.97 & 82.23 & 82.97 \\
    Deepseek-vl2    & 11,126 & 15,984 & 20,426 & 21.75 & 81.12 & 83.52 & 84.53 \\
    \hline
    \rowcolor{gray!10}
    \textbf{Mean / Total} & 146,079 & 207,604 & 265,098 & 21.63 & 82.24 & 83.97 & \textbf{85.10} \\
    \shline
  \end{tabularx}
\end{table*}

\section{Complete Repair Results}
\label{app:counts}
\renewcommand{\thetable}{B\arabic{table}}
\renewcommand{\theHtable}{B.\arabic{table}}
\setcounter{table}{0}

\subsection{Aggregate Results Across Methods}

The core design property of \method{} is that it only replaces the IDs
of known glitch tokens without touching normal-token encodings.
Therefore, \method{}'s RR is 0.00\% on all models: any token originally
classified as normal retains its original encoding after \method{} is
applied, and the model receives input IDs identical to the original.
This $\mathrm{RR}=0.00\%$ is a structural property determined by the
method design: \method{} operates on the canonical tokenization output
and only substitutes positions matching known glitch-token IDs in the
rule table. Normal-token IDs are not in the rule table, so their
encoding process and final IDs remain completely unchanged. This
differs from the baselines: GlitchProber's activation intervention may
affect subsequent token processing, GlitchEdit's embedding editing may
affect the output head under weight tying, and GlitchCleaner's gated
LoRA branch may produce small activations on non-target tokens.

\begin{table}[htbp]
  \centering
  \small
  \setlength{\tabcolsep}{2pt}
  \caption{\textbf{Aggregate repair results across all four methods.}
  Fixed and Reg.\ are summed over the ten models; FR, RR and GR are ten-model averages.}
  \label{tab:aggregate}
  \begin{tabularx}{\columnwidth}{@{}l||>{\hsize=1.1\hsize}C|>{\hsize=.9\hsize}C||>{\hsize=1.1\hsize}C|>{\hsize=.9\hsize}C||>{\hsize=1.0\hsize}C@{}}
    \shline
    \textbf{Method} &
    \textbf{Fixed} & \textbf{FR} &
    \textbf{Reg.} & \textbf{RR} &
    \textbf{GR} \\
    \hline\hline
    GlitchProber  & 63,614  & 24.63 & 382 & 0.024 & 11.37 \\
    GlitchEdit    & 154,382 & 58.39 & 182 & 0.012 & 6.30  \\
    GlitchCleaner & 185,940 & 70.73 & 713 & 0.046 & 4.53  \\
    \hline
    \rowcolor{gray!10}
    \textbf{\method{}} & \textbf{225,197} & \textbf{85.10} &
    \textbf{0} & \textbf{0.00} & \textbf{2.27} \\
    \shline
  \end{tabularx}
\end{table}

\subsection{Per-Model Complete Counts}

Table~\ref{tab:all-counts} gives the complete per-model counts for
each method. Reg.\ and Maint.\ denote regressed and maintained counts; FR, RR, and GR are in percent.

\begin{table*}[t]
  \centering
  \small
  \setlength{\tabcolsep}{4pt}
  \renewcommand{\arraystretch}{1.0}
  \newcommand{\modelrow}[3]{\multicolumn{8}{@{}l}{\textbf{#1}\quad $|G_0|=#2$;\quad $|N_0|=#3$}\\[2pt]}
  \caption{\textbf{Per-model complete counts for all four methods.}
  Glitch and normal counts are in model headings.}
  \label{tab:all-counts}
  \begin{tabularx}{\textwidth}{@{}>{\hsize=1.35\hsize}Y||>{\hsize=.95\hsize}C|>{\hsize=.95\hsize}C|>{\hsize=.8\hsize}C||>{\hsize=.7\hsize}C|>{\hsize=1.05\hsize}C|>{\hsize=.8\hsize}C||>{\hsize=.8\hsize}C@{}}
    \shline
    \textbf{Method} & \textbf{Fixed} & \textbf{Unfixed} &
    \textbf{FR\%} & \textbf{Reg.} & \textbf{Maint.} &
    \textbf{RR\%} & \textbf{GR\%} \\
    \hline\hline
    \modelrow{Llama-8B}{22,026}{106,230}
    GlitchProber  & 4,689  & 17,337 & 21.29 & 32 & 106,198 & 0.03 & 13.54 \\
    GlitchEdit    & 13,031 & 8,995  & 59.16 & 11 & 106,219 & 0.01 & 7.02 \\
    GlitchCleaner & 14,993 & 7,033  & 68.07 & 53 & 106,177 & 0.05 & 5.52 \\
    \rowcolor{gray!10}
    \textbf{\method{}} & \textbf{18,722} & \textbf{3,304} & \textbf{85.00} & \textbf{0} & \textbf{106,230} & \textbf{0.00} & \textbf{2.58} \\
    \hline
    \modelrow{Llama-70B}{15,878}{112,378}
    GlitchProber  & 4,917  & 10,961 & 30.97 & 11 & 112,367 & 0.01 & 8.55 \\
    GlitchEdit    & 10,325 & 5,553  & 65.03 & 11 & 112,367 & 0.01 & 4.34 \\
    GlitchCleaner & 12,461 & 3,417  & 78.48 & 34 & 112,344 & 0.03 & 2.69 \\
    \rowcolor{gray!10}
    \textbf{\method{}} & \textbf{14,455} & \textbf{1,423} & \textbf{91.04} & \textbf{0} & \textbf{112,378} & \textbf{0.00} & \textbf{1.11} \\
    \hline
    \modelrow{Gemma-31B}{36,045}{226,099}
    GlitchProber  & 13,045 & 23,000 & 36.19 & 45 & 226,054 & 0.02 & 8.79 \\
    GlitchEdit    & 23,923 & 12,122 & 66.37 & 23 & 226,076 & 0.01 & 4.63 \\
    GlitchCleaner & 29,297 & 6,748  & 81.28 & 90 & 226,009 & 0.04 & 2.61 \\
    \rowcolor{gray!10}
    \textbf{\method{}} & \textbf{32,080} & \textbf{3,965} & \textbf{89.00} & \textbf{0} & \textbf{226,099} & \textbf{0.00} & \textbf{1.51} \\
    \hline
    \modelrow{Gemma-E4B}{48,287}{213,857}
    GlitchProber  & 5,963  & 42,324 & 12.35 & 86  & 213,771 & 0.04 & 16.18 \\
    GlitchEdit    & 24,998 & 23,289 & 51.77 & 43  & 213,814 & 0.02 & 8.90 \\
    GlitchCleaner & 28,547 & 19,740 & 59.12 & 128 & 213,729 & 0.06 & 7.58 \\
    \rowcolor{gray!10}
    \textbf{\method{}} & \textbf{38,873} & \textbf{9,414} & \textbf{80.50} & \textbf{0} & \textbf{213,857} & \textbf{0.00} & \textbf{3.59} \\
    \hline
    \modelrow{Qwen3.8-27B}{29,451}{218,869}
    GlitchProber  & 7,956  & 21,495 & 27.01 & 44 & 218,825 & 0.02 & 8.67 \\
    GlitchEdit    & 17,813 & 11,638 & 60.48 & 22 & 218,847 & 0.01 & 4.70 \\
    GlitchCleaner & 21,800 & 7,651  & 74.02 & 88 & 218,781 & 0.04 & 3.12 \\
    \rowcolor{gray!10}
    \textbf{\method{}} & \textbf{25,770} & \textbf{3,681} & \textbf{87.50} & \textbf{0} & \textbf{218,869} & \textbf{0.00} & \textbf{1.48} \\
    \hline
    \modelrow{Qwen3.6-27B}{31,040}{217,280}
    GlitchProber  & 9,340  & 21,700 & 30.09 & 43 & 217,237 & 0.02 & 8.76 \\
    GlitchEdit    & 19,142 & 11,898 & 61.67 & 22 & 217,258 & 0.01 & 4.80 \\
    GlitchCleaner & 23,609 & 7,431  & 76.06 & 87 & 217,193 & 0.04 & 3.03 \\
    \rowcolor{gray!10}
    \textbf{\method{}} & \textbf{26,546} & \textbf{4,494} & \textbf{85.52} & \textbf{0} & \textbf{217,280} & \textbf{0.00} & \textbf{1.81} \\
    \hline
    \modelrow{Qwen3.6-A3B}{35,261}{213,059}
    GlitchProber  & 6,195  & 29,066 & 17.57 & 64  & 212,995 & 0.03 & 11.73 \\
    GlitchEdit    & 18,625 & 16,636 & 52.82 & 21  & 213,038 & 0.01 & 6.71 \\
    GlitchCleaner & 22,786 & 12,475 & 64.62 & 107 & 212,952 & 0.05 & 5.07 \\
    \rowcolor{gray!10}
    \textbf{\method{}} & \textbf{29,438} & \textbf{5,823} & \textbf{83.49} & \textbf{0} & \textbf{213,059} & \textbf{0.00} & \textbf{2.34} \\
    \hline
    \modelrow{Mistral-7B}{6,240}{25,760}
    GlitchProber  & 1,308 & 4,932 & 20.96 & 8  & 25,752 & 0.03 & 15.44 \\
    GlitchEdit    & 3,310 & 2,930 & 53.04 & 5  & 25,755 & 0.02 & 9.17 \\
    GlitchCleaner & 4,209 & 2,031 & 67.45 & 15 & 25,745 & 0.06 & 6.39 \\
    \rowcolor{gray!10}
    \textbf{\method{}} & \textbf{5,083} & \textbf{1,157} & \textbf{81.46} & \textbf{0} & \textbf{25,760} & \textbf{0.00} & \textbf{3.62} \\
    \hline
    \modelrow{GLM-4.7}{20,444}{134,436}
    GlitchProber  & 7,296  & 13,148 & 35.69 & 27 & 134,409 & 0.02 & 8.51 \\
    GlitchEdit    & 12,845 & 7,599  & 62.83 & 13 & 134,423 & 0.01 & 4.91 \\
    GlitchCleaner & 15,862 & 4,582  & 77.59 & 67 & 134,369 & 0.05 & 3.00 \\
    \rowcolor{gray!10}
    \textbf{\method{}} & \textbf{16,963} & \textbf{3,481} & \textbf{82.97} & \textbf{0} & \textbf{134,436} & \textbf{0.00} & \textbf{2.25} \\
    \hline
    \modelrow{Deepseek-vl2}{20,426}{108,854}
    GlitchProber  & 2,905  & 17,521 & 14.22 & 22 & 108,832 & 0.02 & 13.57 \\
    GlitchEdit    & 10,370 & 10,056 & 50.77 & 11 & 108,843 & 0.01 & 7.79 \\
    GlitchCleaner & 12,376 & 8,050  & 60.59 & 44 & 108,810 & 0.04 & 6.26 \\
    \rowcolor{gray!10}
    \textbf{\method{}} & \textbf{17,267} & \textbf{3,159} & \textbf{84.53} & \textbf{0} & \textbf{108,854} & \textbf{0.00} & \textbf{2.44} \\
    \shline
  \end{tabularx}
\end{table*}

\section{Formal Properties and Tokenizer Adaptation}
\label{app:properties}
\renewcommand{\thetable}{C\arabic{table}}
\renewcommand{\theHtable}{C.\arabic{table}}
\setcounter{table}{0}

\subsection{Local Equivalence (Proposition 1)}
\label{app:local-equiv}

Let the adapter's local decoding transition be
\begin{equation}
  \mathcal{D}(s, p) = (b, s'),
\end{equation}
where $s$ is the decoder state entering the span, $p$ is the token
sequence, $b$ is the output byte content, and $s'$ is the exit state.
A replacement $p^0 \mapsto p$ is valid when
\begin{equation}
  \mathcal{D}(s, p^0) = \mathcal{D}(s, p),
\end{equation}
and the special-token ID sequences within both spans are identical.

\begin{proposition}[Compatible local replacements preserve global
content]
For a set of non-overlapping replacements whose entry states match and
whose local transitions satisfy decoding equivalence, the final
sequence has the same decoded bytes, special-token ID sequence, and
decoder exit state as the original sequence.
\end{proposition}

\begin{proof}
Process spans in their original order. The first replacement preserves
both its emitted bytes and exit state, so the subsequent unchanged
segment receives the same entry state. Repeat the same argument for
each later replacement. Induction establishes equality of the full
output and terminal state. If adjacent spans cannot be validated
independently, validate their joint region; without a joint rule,
retain the original.
\end{proof}

This proposition covers implementations satisfying the adapter
contract. Additional normalization, text cleanup, or control-structure
handling must be included in the state or fixed beforehand, and
cannot change content outside the contract.

\subsection{Preservation of Untriggered Inputs (Proposition 2)}
\label{app:preservation}

\begin{proposition}[Unchanged effective inputs preserve the model
conditional distribution]
If $R(T(x)) = T(x)$ and the model parameters, attention mask, position
IDs, and decoding configuration are identical, then
\begin{equation}
  P_\theta(y \mid R(T(x))) = P_\theta(y \mid T(x)).
\end{equation}
\end{proposition}

\begin{proof}
Both conditions send identical effective inputs to the same model,
defining the same conditional distribution. Under identical
deterministic execution, greedy outputs are also identical.
\end{proof}

\paragraph{Scope of the preservation guarantee.}
Proposition~2 directly explains the $\mathrm{RR}=0.00\%$ result
observed across all ten models in full-vocabulary testing. Since
\method{}'s rule table contains only glitch-token IDs, when a normal
token appears in the input it matches no rule, so
$R(T(x)) = T(x)$ holds. The model receives input IDs identical to the
original encoding, and the conditional distribution is unchanged.  Baseline methods cannot provide the same
guarantee because their repairs act on model internals (activations,
embeddings, or adapter parameters), and internal modifications can
propagate to non-target token processing through attention mechanisms,
residual connections, or weight tying.

\subsection{Graph Search Guarantee}
\label{app:search}

Every ordinary edge in the candidate graph advances by at least one
byte; empty-content edges and control tokens are excluded. Under fixed
boundaries and finite decoder state, the graph is acyclic. Augmenting
each node with the previous token and used length makes the pairwise
term in Equation~\ref{eq:proxy} additive, enabling budgeted top-$K$
path search on the augmented graph without enumerating every complete
path.

Let $|E'|$ denote the edge count after state augmentation. A
conservative complexity bound for direct implementation is
$O(BK|E'|\log K)$ when length is tracked separately; if $E'$ already
incorporates the length dimension, the $B$ factor should not be counted
again. Actual cost depends on trie matching and the number of reachable
states. The trie is built once and reused across targets.

The graph objective is a proposal cost, not the true
$-\log P_\theta(p \mid c)$. Full model probabilities depend on the
complete prefix and generally cannot be decomposed into fixed local
edge costs on the same graph. Full-context behavioral evaluation is
therefore necessary to assess shortlisted paths.

\subsection{Tokenizer Adapter Contract}
\label{app:adapters}

The evaluated models span two main BPE implementations: tiktoken-style
byte-level BPE (Llama, Qwen, GLM, deepseek-vl2) and SentencePiece BPE
(Gemma, Mistral). All implementations share the candidate graph
construction logic but require separate adaptation for byte sequence
extraction, decoder state management, and leading-space handling.
Tables~\ref{tab:tokenizer-families} and~\ref{tab:adapter-contract}
summarize the adapter targets and contract.

\begin{figure*}[t]
  \centering
  \includegraphics[width=\textwidth]{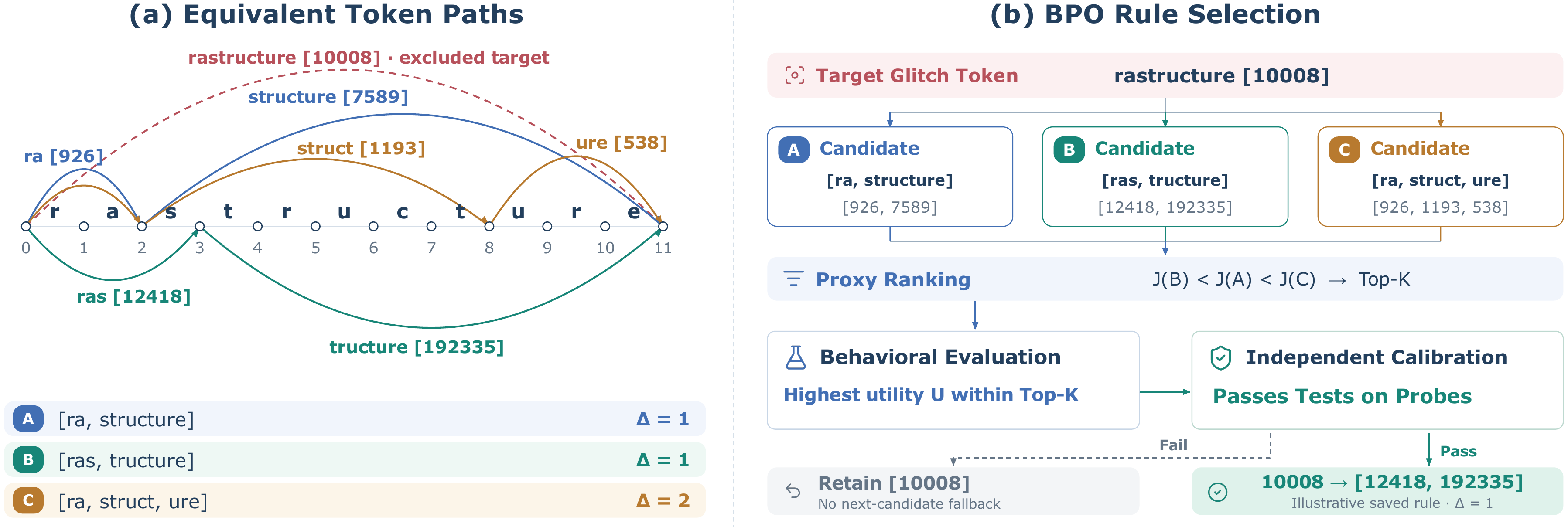}
  \caption{\textbf{Candidate paths and rule selection.}
  (a)~Normal vocabulary entries define different paths over the same
  byte sequence. The original glitch ID is excluded from replacement
  edges. (b)~Proxy ranking, full-sequence evaluation, and independent
  calibration select a rule.}
  \label{fig:bpo-paths}
\end{figure*}

\begin{table}[htbp]
  \centering
  \small
  \setlength{\tabcolsep}{3pt}
  \caption{\textbf{Tokenizer adapter targets per model.} BPE denotes byte-level BPE; SP denotes SentencePiece BPE.}
  \label{tab:tokenizer-families}
  \begin{tabularx}{\columnwidth}{@{}>{\hsize=1.2\hsize}Y||>{\hsize=.55\hsize}Y|>{\hsize=.75\hsize}C|>{\hsize=1.5\hsize}Y@{}}
    \shline
    \textbf{Model} & \textbf{Type} & \textbf{Vocab} & \textbf{Adaptation} \\
    \hline\hline
    Llama-8B      & BPE & 128,256 & Raw byte graph \\
    Llama-70B     & BPE & 128,256 & Shared with 8B \\
    Gemma-31B     & SP  & 262,144 & Normalizer + state \\
    Gemma-E4B     & SP  & 262,144 & Shared with 31B \\
    Qwen3.8-27B   & BPE & 248,320 & Raw byte graph \\
    Qwen3.6-27B   & BPE & 248,320 & Same as above \\
    Qwen3.6-A3B   & BPE & 248,320 & Shared with 27B \\
    Mistral-7B    & SP  & 32,000  & Normalizer + state \\
    GLM-4.7       & BPE & 154,880 & Raw byte graph \\
    Deepseek-vl2  & BPE & 129,280 & Raw byte graph \\
    \shline
  \end{tabularx}
\end{table}

\begin{table}[htbp]
  \centering
  \small
  \setlength{\tabcolsep}{3pt}
  \caption{\textbf{Tokenizer adapter contract.}}
  \label{tab:adapter-contract}
  \begin{tabularx}{\columnwidth}{@{}>{\hsize=.8\hsize}Y||>{\hsize=1.1\hsize}Y|>{\hsize=1.1\hsize}Y@{}}
    \shline
    \textbf{Case} & \textbf{Required handling} &
    \textbf{Invalid shortcut} \\
    \hline\hline
    Byte-level BPE &
    Build edges from exact raw bytes &
    Substitute visible Unicode for incomplete bytes \\
    \hline
    SentencePiece BPE &
    Preserve normalizer, leading space, byte fallback, state &
    Infer algorithm from file format alone \\
    \hline
    Partial UTF-8 &
    Expand to complete verifiable span; compare joint decoding &
    Decode to replacement character and build from it \\
    \hline
    Added/special tokens &
    Preserve IDs, order, control structure; exclude as edges &
    Reconstruct control markers via string matching \\
    \hline
    Normalization &
    Use $T(x)$'s actual byte content as reference &
    Claim preservation without checking \\
    \hline
    Incompatible / over-budget &
    Retain canonical IDs &
    Delete chars, insert spaces, or rewrite text \\
    \shline
  \end{tabularx}
\end{table}

\subsection{BPO Pseudocode and Path Graph}
\label{app:bpo}

Algorithm~\ref{alg:bpo} summarizes the offline construction.
Figure~\ref{fig:bpo-paths} illustrates three candidate paths for
``rastructure''; the IDs shown are tokenizer-specific.

\begin{algorithm}[t]
\caption{\textbf{Behavioral Path Optimization}}
\label{alg:bpo}
\small
\textbf{Input:} tokenizer $T$; glitch set $G_0$; normal set $N_0$;
budget $B$; shortlist size $K$; proxy parameters $(\eta, \lambda_r)$;
utility parameters $(\kappa, \lambda, \gamma)$; check tests $P_{\mathrm{chk}}$; selection probes $P_s$;
calibration probes $P_c$.\par
\textbf{Output:} replacement rule table $\mathrm{Rules}$.\par
\medskip
\begin{enumerate}[leftmargin=1.6em,itemsep=2pt,topsep=0pt]
\item Initialize $\mathrm{Rules}\gets\varnothing$ and an empty shared
  evidence cache $\mathcal{E}$.
  \item \textbf{For each} $g \in G_0$:
\begin{enumerate}[label=\alph*.,leftmargin=1.4em,itemsep=2pt,topsep=2pt]
  \item Evaluate the original representation on $P_s$ to obtain $m_+$.
    If $m_+=0$, retain $g$ and continue to the next target.
  \item Extract byte sequence $b$ of $g$; build directed graph using
    $N_0$ tokens covering $b$.
    \item For each component token or adjacent pair required by the graph
  but absent from $\mathcal{E}$, estimate $\hat{f}(v)$ or
  $\hat{f}(u,v)$ using the supplemental probes and cache the result.
  \item Retrieve up to $K$ paths with the smallest $J(p)$ subject to
    $\Delta(p) \leq B$ and the decoder contract.
  \item Initialize $\mathrm{best} \gets \varnothing$,
    $U_{\mathrm{best}} \gets \gamma$.
  \item \textbf{For each} shortlisted $p$: if $p$ fails any of the
    three check tests $P_{\mathrm{chk}}$, discard it. Otherwise compute
    $U(g,p)$ on $P_s$; if $U(g,p) > U_{\mathrm{best}}$, update
    $\mathrm{best} \gets p$, $U_{\mathrm{best}} \gets U(g,p)$.
  \item If $\mathrm{best}$ exists, run independent calibration on
    $P_c$.
  \item If all calibration probes pass, store the rule (original ID,
    replacement IDs, content, added length, utility). Otherwise retain
    $g$; do not test the next candidate on $P_c$.
\end{enumerate}
\item \textbf{Return} $\mathrm{Rules}$.
\end{enumerate}
\end{algorithm}

\section{Diagnostics, Parameters, and Additional Analyses}
\label{app:diagnostics}
\renewcommand{\thetable}{D\arabic{table}}
\renewcommand{\theHtable}{D.\arabic{table}}
\setcounter{table}{0}

\subsection{Probe Sets and Selection Probe Structure}
\label{app:probesets}

BPO accepts a replacement rule through three stages, each using its
own probe set (Table~\ref{tab:probesets}). All templates are
mutually disjoint and distinct from the GlitchQuiz prompts used for the
reported FR/RR/GR; all probes share the decoding budget of
Appendix~\ref{app:protocol} (greedy decoding, 64 new tokens).

\begin{table*}[t]
  \centering
  \small
  \caption{\textbf{Probe sets used in BPO.}}
  \label{tab:probesets}
  \begin{tabularx}{\textwidth}{@{}l|l||>{\hsize=.7\hsize}C|>{\hsize=1.5\hsize}Y|>{\hsize=.8\hsize}Y@{}}
    \shline
    \textbf{Stage} & \textbf{Probe set} & \textbf{Size} & \textbf{Role} & \textbf{Failure allowed?} \\
    \hline\hline
    Selection   & Selection probes $P_s$ & 12 (4 per dim.) & Compute $U(g,p)$: $N_{\mathrm{fix}}$, $N_{\mathrm{reg}}$, $m_+$ & Yes (penalized via $\kappa$) \\
    \hline
    Validation  & Check tests $P_{\mathrm{chk}}$ & 3 (1 per dim.) & Hard gate: candidate must pass all three & No \\
    \hline
    Calibration & Calibration probes $P_c$ & 4 (1--2 per dim.) & Confirm the selected $p^\star$; failure retains original & No \\
    \hline
    Final evaluation & GlitchQuiz & 3 tasks & Report FR / RR / GR & -- \\
    \shline
  \end{tabularx}
\end{table*}

The three check tests use one direct-format prompt per dimension
(repeat / spell / count), with wordings different from both the
GlitchQuiz templates and the selection variants below.
Table~\ref{tab:probes} describes the 12 selection probes: four prompt
variants per GlitchQuiz dimension, each using a different
question format and context structure to more thoroughly assess
replacement robustness on that dimension. The calibration stage uses
4 independent probes (1--2 per dimension) that do not participate
in the preceding ranking or selection.

\begin{table}[t]
  \centering
  \small
  \setlength{\tabcolsep}{3pt}
  \caption{\textbf{Twelve selection probes: four variants per dimension.}
  $g$ denotes the target string.}
  \label{tab:probes}
  \begin{tabularx}{\columnwidth}{@{}l|l||>{\hsize=1.15\hsize}Y|>{\hsize=.85\hsize}Y@{}}
    \shline
    \textbf{Dim.} & \textbf{Variant} &
    \textbf{Prompt structure} & \textbf{Criterion} \\
    \hline\hline
    \multirow{4}{*}{Rep.}
    & Direct      & Repeat $g$ exactly        & Exact match \\
    & Embedded    & Extract $g$ from sentence  & Exact extraction \\
    & Label bind  & Bind $g$ to label A; query & Exact match \\
    & Selection   & Select $g$ from list       & Correct selection \\
    \hline
    \multirow{4}{*}{Spell}
    & Forward     & List chars of $g$ in order & Correct sequence \\
    & Reverse     & List chars in reverse      & Correct reverse \\
    & Locate      & Give the $k$th char of $g$ & Correct char \\
    & Frequency   & Count char $c$ in $g$      & Correct count \\
    \hline
    \multirow{4}{*}{Length}
    & Direct      & Report length of $g$       & Correct length \\
    & Compare     & Compare $g$ vs.\ $h$       & Correct choice \\
    & Interval    & Is length in $[a,b]$?      & Correct judgment \\
    & Enum+count  & List chars then total      & Both correct \\
    \shline
  \end{tabularx}
\end{table}

\subsection{Default Parameters}
\label{app:parameters}

Table~\ref{tab:parameters} lists all default parameters. Global
parameters are selected on a Qwen3.8-27B development set from a fixed
grid: $K \in \{1,2,4,8,16\}$, $B \in \{4,8,16\}$,
$\eta \in \{0, 0.25, 0.50, 1.00\}$, $\kappa \in \{1,2,4\}$,
$\gamma \in \{0, 0.05, 0.10\}$. The optimization objective is to
maximize net repair gain under limited sequence growth. The selected
configuration is fixed for all other models and the main evaluation.

\begin{table}[t]
  \centering
  \small
  \setlength{\tabcolsep}{3pt}
  \caption{\textbf{Default \method{} parameters.}}
  \label{tab:parameters}
  \begin{tabularx}{\columnwidth}{@{}>{\hsize=1.15\hsize}Y||>{\hsize=.55\hsize}C|>{\hsize=1.3\hsize}Y@{}}
    \shline
    \textbf{Parameter} & \textbf{Value} & \textbf{Meaning} \\
    \hline\hline
    $K$                  & 8         & Max candidate paths per target \\
    $B$                  & 8         & Max added tokens per span \\
    Span content limit   & 64 bytes  & Retain if exceeded \\
    $\eta$               & 0.5       & Adjacent-pair proxy weight \\
    $\lambda_r$          & 0.02      & Proxy length cost \\
    $\kappa$             & 2         & Regression penalty coefficient \\
    $\lambda$            & 0.02      & Utility length cost \\
    $\gamma$             & 0.05      & Minimum selection utility \\
    Shared-pool probes   & 3/entry   & For components and pairs \\
    Check tests          & 3         & Hard gate, 1 per dimension \\
    Selection / calib.   & 12 / 4    & Per-target contexts \\
    Probe output limit   & 64 tokens & Matching detection budget \\
    \shline
  \end{tabularx}
\end{table}

\subsection{Baseline Adaptation Configurations}
\label{app:baselines}

Table~\ref{tab:baselines} describes the construction and adaptation
settings for each baseline. GlitchCleaner uses AdamW with effective
batch size 32, maximum sequence length 512, warmup ratio 0.03, and zero
weight decay. GlitchEdit follows the model's actual parameter sharing
when embeddings are tied, without silently untying them. \method{}'s
rule table only changes input IDs, so the model's parameter sharing
remains intact.

\begin{table}[htbp]
  \centering
  \small
  \setlength{\tabcolsep}{2pt}
  \caption{\textbf{Baseline construction and adaptation.}}
  \label{tab:baselines}
  \begin{tabularx}{\columnwidth}{@{}>{\hsize=.78\hsize}Y||>{\hsize=1.27\hsize}Y|>{\hsize=.95\hsize}Y@{}}
    \shline
    \textbf{Method} & \textbf{Mechanism \& config} &
    \textbf{Parameter selection} \\
    \hline\hline
    GlitchProber &
    Activation statistics + neuron correction;
    2,000 ref.\ tokens; threshold 0.99; later MLP layers &
    Relative layer interval grid;
    ${\sim}[0.60L, 0.90L]$ \\
    \hline
    GlitchEdit &
    Target embedding mixed with normal neighbors;
    $\alpha \in [0,1]$, step 0.1; 1--10 neighbors &
    Behavioral acceptance; retain failure fallback \\
    \hline
    GlitchCleaner &
    Target-triggered MLP gated LoRA; rank 4,
    scaling 4, lr $3{\times}10^{-4}$ &
    Freeze base weights; independent calib.\
    from $\leq$20 epochs \\
    \shline
  \end{tabularx}
\end{table}

\subsection{Parameter and Component Ablations}
\label{app:ablations}

The parameter sweeps in Table~\ref{tab:ablation-sweeps}
use Qwen3.8-27B with $K=8$. Ablation experiments in Table~\ref{tab:components-three} are conducted on three models using each model's full glitch set.

\begin{table}[htbp]
  \centering
  \small
  \caption{\textbf{Parameter sweeps (Qwen3.8-27B, $K=8$).}}
  \label{tab:ablation-sweeps}
  \begin{subtable}{\columnwidth}
    \centering
    \caption{Utility threshold $\gamma$.}
    \label{tab:ablation-gamma}
    \begin{tabularx}{\columnwidth}{@{}>{\hsize=1.3\hsize}Y||>{\hsize=.9\hsize}C|>{\hsize=.9\hsize}C|>{\hsize=.9\hsize}C@{}}
      \shline
      $\gamma$ & 0.00 & 0.05 & 0.10 \\
      \hline\hline
      FR (\%) & 90.23 & 87.50 & 80.86 \\
      Coverage (\%) & 94.8 & 89.2 & 81.8 \\
      \shline
    \end{tabularx}
  \end{subtable}
  \par\medskip
  \begin{subtable}{\columnwidth}
    \centering
    \caption{Adjacent-pair weight $\eta$.}
    \label{tab:ablation-eta}
    \begin{tabularx}{\columnwidth}{@{}>{\hsize=1.3\hsize}Y||>{\hsize=.925\hsize}C|>{\hsize=.925\hsize}C|>{\hsize=.925\hsize}C|>{\hsize=.925\hsize}C@{}}
      \shline
      $\eta$ & 0.00 & 0.25 & 0.50 & 1.00 \\
      \hline\hline
      FR (\%) & 81.73 & 85.03 & 87.50 & 86.47 \\
      \shline
    \end{tabularx}
  \end{subtable}
  \par\medskip
  \begin{subtable}{\columnwidth}
    \centering
    \caption{Span token budget $B$.}
    \label{tab:ablation-budget}
    \begin{tabularx}{\columnwidth}{@{}>{\hsize=1.3\hsize}Y||>{\hsize=.9\hsize}C|>{\hsize=.9\hsize}C|>{\hsize=.9\hsize}C@{}}
      \shline
      $B$ & 4 & 8 & 16 \\
      \hline\hline
      FR (\%) & 82.48 & 87.50 & 88.52 \\
      Mean growth (\%) & 1.8 & 2.6 & 4.1 \\
      \shline
    \end{tabularx}
  \end{subtable}
\end{table}

\paragraph{Utility threshold $\gamma$.}
Table~\ref{tab:ablation-gamma} shows the effect of $\gamma$.
Here, \emph{coverage} denotes the fraction of the initial glitch set
$G_0$ for which a replacement rule is stored after candidate selection
and independent calibration, whereas FR is measured afterwards using
the disjoint final GlitchQuiz evaluation. Coverage therefore upper-bounds
FR but need not equal it: a stored rule may pass selection and
calibration yet still fail to convert the target into a normal token
under the final evaluation. Lowering $\gamma$ only relaxes the utility threshold; the three check
tests and independent calibration still apply. At $\gamma=0$, coverage
rises to 94.8\%, close to the fraction of Qwen3.8-27B
targets with at least one feasible path and $m_+>0$, and FR rises to 90.23\%.
The default $\gamma=0.05$ yields 89.2\% coverage and 87.50\% FR; the
approximately 1.7\,pp difference corresponds to targets that receive a
stored rule but do not pass the final evaluation. At $\gamma=0.10$,
coverage falls to 81.8\% and FR to 80.86\%.

\paragraph{Adjacent-pair weight $\eta$.}
Table~\ref{tab:ablation-eta} shows the effect of $\eta$. At $\eta=0$,
adjacency compatibility is ignored, yielding 81.73\% FR. The
default $\eta=0.50$ improves FR to 87.50\%, indicating that moderate
pairwise information helps distinguish more reliable replacement paths.
At $\eta=1.00$, however, excessive emphasis on pairwise risk can
down-rank paths whose sub-tokens are individually better, reducing FR
to 86.47\%. This suggests that pairwise compatibility is useful as a
complementary signal but should not dominate the path-level score.

\paragraph{Span token budget $B$.}
Table~\ref{tab:ablation-budget} shows the effect of $B$. At $B=4$,
some glitch tokens requiring finer decomposition cannot find effective
paths within budget, yielding 82.48\% FR. The difference between
$B=8$ (default) and $B=16$ is only 1.02\,pp, but the mean token
increase rises from 2.6\% to 4.1\%, adding online sequence length cost, suggesting $B=8$ provides a better balance.

\paragraph{The role of $\kappa$.}
$\kappa$ controls the penalty strength for probe-level behavioral
regression during path selection, which is distinct from the
vocabulary-level RR metric. Regressions are counted on the 12
selection probes only, so $\kappa$ re-ranks candidates that have
already passed the three check tests and never admits one that fails
them. Because the gain and regression terms use
different denominators, a regressed probe has $\kappa m/m_+$ times the
weight of a repaired probe when $m_+>0$. Thus, at $\kappa=1$, a
regressed probe is at least as costly as a repaired probe; at
$\kappa=2$, it is at least twice as costly, steering the
algorithm toward behaviorally stable paths. On Qwen3.8-27B, 61.3\%
of targets have at least two candidates passing the check tests; among
these, moving from $\kappa=1$ to $\kappa=2$ changes the selected path
for 7.4\% and the final repair outcome for 1.2\% of targets (FR 87.18\%
vs.\ 87.50\%). At $\kappa=4$, the stronger
penalty over-excludes candidates with positive net benefit and
reduces FR to 86.71\%.

\paragraph{Candidate count $K$ (Table~\ref{tab:ablation-k}).}
$K$ is the number of candidate paths retained per target after proxy
ranking. Build cost is estimated as
$C(K)=C_{\mathrm{shared}}+N_{\mathrm{target}}\cdot K$, where
$C_{\mathrm{shared}}$ denotes fixed preparation overhead (supplemental
probe evaluations for the shared evidence pool) and
$N_{\mathrm{target}}\cdot K$ the total candidate evaluations. From
$K=1$ to $K=8$, FR increases from 71.83\% to 87.50\%, with successive
gains of 5.63, 5.45, and 4.59\,pp. Further increasing to $K=16$ yields
only 0.43\,pp additional gain while nearly doubling the build cost. This
shows diminishing returns, so we adopt $K=8$ as the default balancing
repair effectiveness and evaluation budget.

\paragraph{Component ablation (Table~\ref{tab:ablation-component}).}
Restricting search to the BPE merge history reduces FR to 73.82\%
($-$13.68\,pp), showing the value of vocabulary-wide paths. Removing the
pair term ($\eta{=}0$) lowers FR to 81.73\%, confirming the benefit of
adjacent-token compatibility, while removing full-path verification
yields 85.34\%, a further 2.16\,pp drop. Setting $\gamma{=}0$ raises FR
to 90.23\% but admits marginal paths prone to regressions, so the
default $\gamma{=}0.05$ trades 2.73\,pp for stability. Thus all
four components contribute to repair quality.

\paragraph{Fixed local decompositions (Table~\ref{tab:local-baselines}).}
Since \method{} replaces only glitch-token IDs, a natural question is
how much BPO adds over a fixed local rule applied to the same IDs. We
evaluate four such rules (and Local-Char/Byte additionally without the
budget) under the exact setting of \method{} (online substitution of
glitch-token IDs only, the same length budget $B=8$, and the protocol of
Section~\ref{sec:repair}): Local-Depth1/2/3 replace each glitch token
with its sub-tokens at the corresponding depth of its merge tree (a
child that is itself a glitch token is kept), and Local-Char/Byte
decomposes it into single characters (SentencePiece) or single bytes
(byte-level BPE). Local-Depth1 fixes 38.38--43.53\%, close to the global
depth-1 FR of Section~\ref{sec:empirical} but without collateral
regressions; deeper rules do not help monotonically (Finding~2), and
character/byte-level decomposition fixes only 31.84--36.01\% because it
exceeds the budget for most glitch tokens, and 47.58--53.89\% even
without the budget. \method{} fixes 85.00--89.00\%: the gain comes from
selecting, per token, among vocabulary-wide equivalent paths by
behavior.

\paragraph{Three-model component ablation.}
Table~\ref{tab:components-three} extends the component ablation to
three models. All three show consistent trends: BPE-history-only search
yields the lowest FR, lowering the utility threshold to $\gamma=0$
gives the highest FR but weaker behavioral guarantees, and the full method
balances FR with reliability. Gemma-4-31B has slightly higher FR across
all settings, possibly because its larger vocabulary provides more
equivalent paths.

\begin{table}[htbp]
  \centering
  \small
  \setlength{\tabcolsep}{2pt}
  \caption{\textbf{Three-model component ablation (FR\%).}}
  \label{tab:components-three}
  \begin{tabularx}{\columnwidth}{@{}>{\hsize=1.42\hsize}Y||>{\hsize=.84\hsize}C|>{\hsize=.74\hsize}C|>{\hsize=1.0\hsize}C@{}}
    \shline
    \textbf{Setting} &
    \textbf{Llama-8B} & \textbf{Qwen3.8} & \textbf{Gemma-31B} \\
    \hline\hline
    Full \method{}      & 85.00 & 87.50 & 89.00 \\
    BPE history only    & 72.41 & 73.82 & 76.53 \\
    Remove pair term    & 79.27 & 81.73 & 84.18 \\
    Remove path eval.   & 82.86 & 85.34 & 87.12 \\
    Threshold $\gamma=0$ & 88.16 & 90.23 & 91.74 \\
    \shline
  \end{tabularx}
\end{table}

\subsection{Proxy Score Quality Analysis}
\label{app:proxy}

\begin{table}[htbp]
  \centering
  \small
  \caption{\textbf{Proxy score vs.\ true utility.}}
  \label{tab:proxy}
  \begin{tabularx}{0.62\columnwidth}{@{}Y||C@{}}
    \shline
    \textbf{Model} & Mean $\rho(-J,U)$ \\
    \hline\hline
    Llama-3.1-8B & 0.72 \\
    Qwen3.8-27B & 0.68 \\
    Gemma-4-31B & 0.71 \\
    \shline
  \end{tabularx}
\end{table}

Table~\ref{tab:proxy} reports the mean per-target Spearman correlation
$\rho(-J,U)$. For each target with at least two candidates in the
separately evaluated diagnostic set, the correlation is computed
between $-J(p)$ and $U(g,p)$; targets with fewer than two evaluated
candidates are excluded. The mean values range from 0.68 to 0.72,
indicating that lower-$J$ candidates tend to have higher utility.
Increasing $K$ from 8 to 16 improves FR by only 0.43\,pp while nearly
doubling candidate-evaluation cost, indicating diminishing returns
beyond $K=8$.

\section{Runtime and Cost}
\label{app:runtime}
\renewcommand{\thetable}{E\arabic{table}}
\renewcommand{\theHtable}{E.\arabic{table}}
\setcounter{table}{0}

\subsection{Per-Model Serving Latency}

Table~\ref{tab:runtime} shows \method{}'s online latency and offline
build time for each model. \method{}'s online overhead comprises
three parts: lightweight host-side rule matching and sequence
reconstruction on every request, and, on triggered requests only, the
additional prefill computation for the inserted tokens and a slightly
longer KV cache during decoding.
The language-model inference itself remains GPU-based. Table~\ref{tab:runtime}
is measured with 512-token prompts, 128 generated tokens, batch size 1,
and a trigger density of 10\% (one request in ten contains a
rule-matched glitch token); under these conditions the mean end-to-end
latency increase across models is 0.065\%, well below 0.1\%.
Table~\ref{tab:latency-decomp} separates the overhead on triggered
requests into its components and reports it at several trigger
densities and batch sizes: the number of decoded tokens is unchanged,
but the longer KV cache adds 0.14--0.23\% to decoding time on triggered
requests.
Offline build times range from 0.36\,h (Mistral-7B) to 2.08\,h
(Llama-3.1-70B), with a mean of 0.89\,h.

\begin{table}[htbp]
  \centering
  \small
  \setlength{\tabcolsep}{2pt}
  \caption{\textbf{\method{} per-model serving latency and offline build time.}
  TP is the tensor-parallel degree; Host is the host-side preparation time.}
  \label{tab:runtime}
  \begin{tabularx}{\columnwidth}{@{}l||>{\hsize=.6\hsize}C|>{\hsize=1.1\hsize}C|>{\hsize=1.1\hsize}C|>{\hsize=1.1\hsize}C|>{\hsize=1.1\hsize}C@{}}
    \shline
    \textbf{Model} & \textbf{TP} &
    \makecell{\textbf{Base}\\\textbf{(s)}} & \makecell{\textbf{Incr.}\\\textbf{(\%)}} &
    \makecell{\textbf{Host}\\\textbf{(ms)}} & \makecell{\textbf{Build}\\\textbf{(h)}} \\
    \hline\hline
    Llama-3.1-8B    & 1 & 1.280 & 0.071 & 0.231 & 0.61 \\
    Llama-3.1-70B   & 4 & 4.600 & 0.058 & 0.247 & 2.08 \\
    Gemma-4-31B     & 2 & 2.950 & 0.069 & 0.258 & 1.27 \\
    Gemma-4-E4B     & 1 & 1.100 & 0.061 & 0.219 & 0.49 \\
    Qwen3.8-27B     & 2 & 2.750 & 0.079 & 0.243 & 1.02 \\
    Qwen3.6-27B     & 2 & 2.800 & 0.068 & 0.238 & 0.96 \\
    Qwen3.6-A3B     & 2 & 1.850 & 0.062 & 0.214 & 0.74 \\
    Mistral-7B      & 1 & 0.950 & 0.051 & 0.187 & 0.36 \\
    GLM-4.7         & 1 & 1.500 & 0.059 & 0.221 & 0.56 \\
    Deepseek-vl2    & 2 & 2.200 & 0.072 & 0.229 & 0.84 \\
    \hline
    \rowcolor{gray!10}
    \textbf{Mean}   & -- & -- & \textbf{0.065} &
    \textbf{0.229} & \textbf{0.89} \\
    \shline
  \end{tabularx}
\end{table}

\subsection{Build Time Comparison}

Table~\ref{tab:build} compares the offline build times of all four
methods across ten models. \method{}'s average build time is 0.89\,h,
approximately 41\% of GlitchProber (2.19\,h), 21\% of GlitchEdit
(4.28\,h), and 9\% of GlitchCleaner (10.41\,h). GlitchCleaner
requires training gated LoRA parameters and has the longest build time,
reaching 22.50\,h on Llama-70B. GlitchProber's build time includes
computing activation statistics for reference tokens.

\begin{table}[htbp]
  \centering
  \small
  \setlength{\tabcolsep}{3pt}
  \caption{\textbf{Offline build time in hours.}}
  \label{tab:build}
  \begin{tabularx}{\columnwidth}{@{}l||C|C|C||C@{}}
    \shline
    \textbf{Model} &
    \textbf{Prober} & \textbf{Edit} & \textbf{Cleaner} &
    \textbf{Ours} \\
    \hline\hline
    Llama-3.1-8B    & 0.80 & 2.80 & 7.20  & \textbf{0.61} \\
    Llama-3.1-70B   & 4.80 & 9.80 & 22.50 & \textbf{2.08} \\
    Gemma-4-31B     & 2.80 & 5.20 & 12.80 & \textbf{1.27} \\
    Gemma-4-E4B     & 2.00 & 4.00 & 9.80  & \textbf{0.49} \\
    Qwen3.8-27B     & 2.50 & 4.50 & 11.50 & \textbf{1.02} \\
    Qwen3.6-27B     & 2.40 & 4.30 & 11.00 & \textbf{0.96} \\
    Qwen3.6-A3B     & 2.00 & 3.80 & 9.20  & \textbf{0.74} \\
    Mistral-7B      & 0.80 & 1.60 & 3.80  & \textbf{0.36} \\
    GLM-4.7         & 1.60 & 3.00 & 7.50  & \textbf{0.56} \\
    Deepseek-vl2    & 2.20 & 3.80 & 8.80  & \textbf{0.84} \\
    \hline
    \rowcolor{gray!10}
    \textbf{Mean}   & 2.19 & 4.28 & 10.41 & \textbf{0.89} \\
    \shline
  \end{tabularx}
\end{table}

\subsection{Combined Cost Summary}

Table~\ref{tab:cost-summary} summarizes the offline and online costs
of all four methods. \method{} has the shortest average offline build
time (0.89\,h) while keeping the mean online latency increase below
0.1\%. It requires no additional model parameters or extra model calls.
GlitchProber has the highest online latency increase ($\sim$9.20\%)
because every inference requires activation monitoring and correction.
GlitchEdit introduces no additional serving latency after editing, but
requires a substantially longer offline build stage. GlitchCleaner has
the longest build time and introduces $\sim$1.00\% additional latency
from the adapter branch.

\begin{table}[htbp]
  \centering
  \small
  \setlength{\tabcolsep}{3pt}
  \caption{\textbf{Combined cost summary (ten-model average).}
  $^\dagger$At 10\% trigger density; 0.33--0.71\% on triggered requests (Table~\ref{tab:latency-decomp}).}
  \label{tab:cost-summary}
  \begin{tabularx}{\columnwidth}{@{}l||>{\hsize=.75\hsize}C|>{\hsize=.95\hsize}C|>{\hsize=1.3\hsize}C|>{\hsize=1.0\hsize}C@{}}
    \shline
    \textbf{Method} &
    \makecell{\textbf{Build}\\\textbf{(h)}} & \makecell{\textbf{Latency}\\\textbf{incr.}} &
    \makecell{\textbf{Extra}\\\textbf{params}} & \makecell{\textbf{Extra}\\\textbf{calls}} \\
    \hline\hline
    GlitchProber  & 2.19  & $\sim$9.20\% & 0         & 0 \\
    GlitchEdit    & 4.28  & 0.00\%       & 0         & 0 \\
    GlitchCleaner & 10.41 & $\sim$1.00\% & LoRA & 0 \\
    \hline
    \rowcolor{gray!10}
    \textbf{\method{}} & \textbf{0.89} & \textbf{$<$0.1\%$^\dagger$} &
    \textbf{0} & \textbf{0} \\
    \shline
  \end{tabularx}
\end{table}

\subsection{Latency on Triggered Requests}
\label{app:latency-decomp}

Table~\ref{tab:latency-decomp} decomposes the online latency increase
of Table~\ref{tab:runtime} on the three representative models. The
host-side component is incurred on every request, whereas the prefill
and decode components arise only on triggered requests. The
``Overall @10\%'' column reproduces the Incr.\ column of
Table~\ref{tab:runtime}; the triggered-only totals rise by at most
0.05\,pp at batch sizes 8 and 32.

\begin{table*}[t]
  \centering
  \small
  \setlength{\tabcolsep}{4pt}
  \caption{\textbf{Decomposition of \method{}'s online latency
  increase (\%) with 512-token prompts and 128 generated tokens.}
  ``Triggered-only'' refers to requests in which at least one rule
  fires; ``Overall'' is the mean over all requests at the given trigger
  density (batch size 1 unless stated); the 10\% column equals Incr.\ in Table~\ref{tab:runtime}.}
  \label{tab:latency-decomp}
  \begin{tabularx}{\textwidth}{@{}l||C|C|C|C||C|C|C||C|C@{}}
    \shline
    & \multicolumn{4}{c||}{\textbf{Triggered-only, batch 1}}
    & \multicolumn{3}{c||}{\textbf{Overall at trigger density}}
    & \multicolumn{2}{c}{\textbf{Triggered-only total}} \\
    \textbf{Model} &
    \textbf{Host} & \textbf{Prefill} & \textbf{Decode} & \textbf{Total} &
    \textbf{1\%} & \textbf{5\%} & \textbf{10\%} &
    \textbf{batch 8} & \textbf{batch 32} \\
    \hline\hline
    Mistral-7B    & 0.020 & 0.178 & 0.135 & \textbf{0.333} & 0.023 & 0.035 & 0.051 & 0.35 & 0.38 \\
    Qwen3.8-27B   & 0.009 & 0.472 & 0.230 & \textbf{0.711} & 0.016 & 0.044 & 0.079 & 0.73 & 0.76 \\
    Llama-3.1-70B & 0.005 & 0.318 & 0.208 & \textbf{0.531} & 0.011 & 0.032 & 0.058 & 0.55 & 0.58 \\
    \shline
  \end{tabularx}
\end{table*}

\section{General Capability and Repair Boundaries}
\label{app:general}
\renewcommand{\thetable}{F\arabic{table}}
\renewcommand{\theHtable}{F.\arabic{table}}
\setcounter{table}{0}

\subsection{General Benchmarks}

Table~\ref{tab:benchmark-counts} shows the correct-answer counts of
tested models on four general benchmarks (MMLU-Pro 1,000 items,
ARC-Challenge 1,000 items, GSM8K 1,000 items, HumanEval 164 items).
\method{} and the original models use identical test samples and decoding settings. Across models and tasks,
the correct counts remain nearly unchanged, with small sample-level
fluctuations, indicating that local ID replacement has negligible
aggregate impact on general performance. Any observed
difference can arise only on benchmark instances whose
canonical tokenization contains at least one rule-matched glitch-token
ID; rule-unmatched instances remain exactly unchanged by construction.
Table~\ref{tab:triggered} quantifies this: across the four
benchmarks and ten models, on average 1.9\% of instances contain a
detected glitch token and 1.5\% trigger a rule. On the 489
triggered instances, \method{} changes the outcome of 7 ($+$5\,/\,$-$2;
accuracy 66.1\% $\to$ 66.7\%). The gains are instances whose answer
hinged on the glitch token; the two losses are borderline items that
flip under the small change of the input sequence. For the remaining
482, the outcome is unchanged either way. Where answers do hinge on the
token, as in the Entity QA, Verbatim Copy and Code Identifier tasks of
Appendix~\ref{app:ood}, the original models fail on 70.7\% of inputs
and \method{} gains $+$22.6\,pp. The repair thus changes outcomes where
glitch tokens matter and otherwise leaves them unchanged.

\begin{table}[htbp]
  \centering
  \small
  \setlength{\tabcolsep}{1.5pt}
  \renewcommand{\arraystretch}{1.0}
  \caption{\textbf{Original vs.\ \method{} (Ours) correct counts on general
  benchmarks.}}
  \label{tab:benchmark-counts}
  \begin{tabularx}{\columnwidth}{@{}l|l||>{\hsize=1.1\hsize}C|>{\hsize=.6\hsize}C|>{\hsize=1.1\hsize}C|>{\hsize=1.2\hsize}C@{}}
    \shline
    \textbf{Model} & \textbf{Method} &
    \makecell{\textbf{MMLU-}\\\textbf{Pro}} & \textbf{ARC} & \textbf{GSM8K} &
    \makecell{\textbf{Human}\\\textbf{Eval}} \\
    \hline\hline
    \multirow{2}{*}{Llama-8B}
    & Original    & 540 & 790 & 805 & 86  \\
    & Ours        & 540 & 791 & 805 & 86  \\
    \hline
    \multirow{2}{*}{Llama-70B}
    & Original    & 660 & 875 & 922 & 121 \\
    & Ours        & 660 & 875 & 922 & 121 \\
    \hline
    \multirow{2}{*}{Gemma-31B}
    & Original    & 640 & 860 & 910 & 118 \\
    & Ours        & 641 & 860 & 910 & 118 \\
    \hline
    \multirow{2}{*}{Gemma-E4B}
    & Original    & 480 & 740 & 760 & 72  \\
    & Ours        & 479 & 740 & 760 & 72  \\
    \hline
    \multirow{2}{*}{Qwen3.8}
    & Original    & 630 & 850 & 900 & 115 \\
    & Ours        & 631 & 850 & 900 & 115 \\
    \hline
    \multirow{2}{*}{Qwen3.6-27B}
    & Original    & 620 & 845 & 895 & 113 \\
    & Ours        & 620 & 845 & 896 & 113 \\
    \hline
    \multirow{2}{*}{Qwen3.6-A3B}
    & Original    & 580 & 815 & 855 & 98  \\
    & Ours        & 580 & 815 & 855 & 98  \\
    \hline
    \multirow{2}{*}{Mistral-7B}
    & Original    & 500 & 770 & 775 & 80  \\
    & Ours        & 500 & 771 & 775 & 80  \\
    \hline
    \multirow{2}{*}{GLM-4.7}
    & Original    & 550 & 800 & 830 & 95  \\
    & Ours        & 550 & 800 & 830 & 95  \\
    \hline
    \multirow{2}{*}{Deepseek-vl2}
    & Original    & 520 & 785 & 815 & 88  \\
    & Ours        & 519 & 785 & 815 & 88  \\
    \shline
  \end{tabularx}
\end{table}

\begin{table*}[t]
  \centering
  \small
  \caption{\textbf{General-benchmark instances containing glitch
  tokens, aggregated over the ten models.} Accuracies are on the
  triggered subset; Gain\,/\,Loss count triggered instances flipped
  wrong$\to$right\,/\,right$\to$wrong.}
  \label{tab:triggered}
  \begin{tabularx}{\textwidth}{@{}l||C|C||C|C|C@{}}
    \shline
    \textbf{Benchmark} & \textbf{Contain} & \textbf{Trigger} &
    \textbf{Orig.\ acc.\ (\%)} & \textbf{\method{} acc.\ (\%)} & \textbf{Gain\,/\,Loss} \\
    \hline\hline
    MMLU-Pro  & 268 (2.7\%) & 221 (2.2\%) & 57.9 & 57.9 & $+$2\,/\,$-$2 \\
    ARC       & 142 (1.4\%) & 118 (1.2\%) & 75.4 & 77.1 & $+$2\,/\,0 \\
    GSM8K     & 61 (0.6\%)  & 49 (0.5\%)  & 73.5 & 75.5 & $+$1\,/\,0 \\
    HumanEval & 124 (7.6\%) & 101 (6.2\%) & 69.3 & 69.3 & 0\,/\,0 \\
    \hline
    \rowcolor{gray!10}
    \textbf{All} & 595 (1.9\%) & 489 (1.5\%) & 66.1 & 66.7 & $+$5\,/\,$-$2 \\
    \shline
  \end{tabularx}
\end{table*}

Table~\ref{tab:general-scores} shows the macro-averaged General Score
computed from Table~\ref{tab:benchmark-counts}, with the General Score (GS) defined as
\begin{equation}
  \mathrm{GS} = 25\left(
    \tfrac{c_1}{1000} + \tfrac{c_2}{1000} +
    \tfrac{c_3}{1000} + \tfrac{c_4}{164}\right),
\end{equation}
where $c_1,\ldots,c_4$ are the correct counts on MMLU-Pro, ARC-Challenge, GSM8K, and HumanEval.
The ten-model $\Delta$ ranges from $-$0.025 to $+$0.025\,pp, with a
mean change of only $+$0.008\,pp. Two models show minimal decreases (both caused by a
single-item change). These small fluctuations indicate no systematic degradation of general benchmark performance.

\begin{table}[htbp]
  \centering
  \small
  \setlength{\tabcolsep}{3pt}
  \caption{\textbf{General Score and change ($\Delta$ in pp).}}
  \label{tab:general-scores}
  \begin{tabularx}{\columnwidth}{@{}l||C|C|C@{}}
    \shline
    \textbf{Model} & \textbf{Original} & \textbf{Ours} &
    \textbf{$\Delta$ (pp)} \\
    \hline\hline
    Llama-3.1-8B      & 66.48 & 66.51 & $+$0.025 \\
    Llama-3.1-70B     & 79.87 & 79.87 & $+$0.000 \\
    Gemma-4-31B       & 78.24 & 78.26 & $+$0.025 \\
    Gemma-4-E4B       & 60.48 & 60.45 & $-$0.025 \\
    Qwen3.8-27B       & 77.03 & 77.06 & $+$0.025 \\
    Qwen3.6-27B       & 76.23 & 76.25 & $+$0.025 \\
    Qwen3.6-35B-A3B   & 71.19 & 71.19 & $+$0.000 \\
    Mistral-7B        & 63.32 & 63.35 & $+$0.025 \\
    GLM-4.7           & 68.98 & 68.98 & $+$0.000 \\
    Deepseek-vl2      & 66.41 & 66.39 & $-$0.025 \\
    \hline
    \rowcolor{gray!10}
    \textbf{Mean}     & \textbf{70.82} & \textbf{70.83} &
    \textbf{$+$0.008} \\
    \shline
  \end{tabularx}
\end{table}

\subsection{Failure and Retention Cases}
Table~\ref{tab:failures} lists the possible failure and retention cases
for \method{}. Taking Qwen3.8-27B as an example, of 29,451 initial
glitch tokens, 25,770 (87.50\%) are successfully repaired and 1,473
(5.00\%) are structurally unrepairable because no budget-feasible
equivalent path exists in the vocabulary. The remaining 2,208 tokens
(7.50\%) have at least one feasible path but remain unrepaired for
several reasons. First, approximately 0.1\% of the initial targets
have $m_+=0$ (all selection probes fail on the original representation),
so no utility comparison is possible and the original encoding is
retained (Algorithm~\ref{alg:bpo}, step~2a). Second, approximately
5.6\% receive no stored rule because the utility of their best candidate
does not exceed the default threshold ($0 < U(g,p) \le \gamma$).
Third, independent calibration rejected the selected candidate for
approximately 0.1\% of targets, confirming that the calibration
stage provides a meaningful guard against overfitting to the selection
probes. Fourth, approximately 1.7\% of the initial targets receive a
stored rule but do not pass the disjoint final GlitchQuiz evaluation,
accounting for the difference between the default rule coverage
(89.2\%, Table~\ref{tab:ablation-gamma}) and the final FR (87.50\%).
On the subset with feasible paths and $m_+>0$, the final FR is
92.21\%.

\begin{table}[htbp]
  \centering
  \small
  \setlength{\tabcolsep}{3pt}
  \caption{\textbf{Failure and retention cases.}}
  \label{tab:failures}
  \begin{tabularx}{\columnwidth}{@{}>{\hsize=1.0\hsize}Y||>{\hsize=1.15\hsize}Y|>{\hsize=.85\hsize}Y@{}}
    \shline
    \textbf{Case} & \textbf{Observable evidence} &
    \textbf{System action} \\
    \hline\hline
    No budget-feasible path &
    Graph search: path count $=0$ &
    Retain original \\
    \hline
    All probes fail on original ($m_+=0$) &
    Selection probes: $m_+=0$ &
    Retain original \\
    \hline
    Unreliable composition &
    Check-test failure &
    Reject candidate \\
    \hline
    Insufficient utility &
    $U(g,p)\leq\gamma$ &
    Retain original \\
    \hline
    Calibration failure &
    Independent probe failure &
    Reject rule \\
    \hline
    Per-span limit &
    $\Delta(p)>B$ or content exceeds 64 bytes &
    Retain original \\
    \hline
    Unseen target &
    No key match &
    Pass through \\
    \hline
    Output-side failure &
    Input IDs unchanged &
    Run original model \\
    \hline
    Final-evaluation failure
    & Stored rule; final GlitchQuiz fails after substitution
    & Rule applied; counted as unfixed \\
    \shline
  \end{tabularx}
\end{table}

\section{Language-Modeling Loss Comparison}
\label{app:loss}
\renewcommand{\thetable}{G\arabic{table}}
\renewcommand{\theHtable}{G.\arabic{table}}
\setcounter{table}{0}

This appendix evaluates whether repair improves the model's
prediction on natural text that contains glitch tokens. We measure two
objectives that remain directly comparable before and after repair.

\paragraph{Protocol.}
For each of the ten models, 1{,}000 glitch tokens are sampled uniformly
from all detected glitch tokens, without conditioning on whether any
method repairs them. For each target, GPT-5.6-Terra generates a natural
passage of approximately 200 tokens that contains the target string
once, within the first half of the passage. Each passage is fed to the
frozen model under the canonical tokenization and under each method's
repaired condition; for \method{}, targets without a stored rule enter
unchanged and contribute zero change. Two objectives are computed:
\ding{182}~\textbf{Byte-normalized NLL}: the negative log-likelihood
of the token sequence the model actually receives, divided by the
passage's UTF-8 byte length, so that the denominator is
tokenization-invariant and the two conditions are compared on the same
byte string.
\ding{183}~\textbf{Identical-continuation NLL}: the NLL summed over the
32 tokens immediately following the target span. Because \method{}
substitutes only matched glitch-token IDs and preserves all other
canonical IDs, these 32 tokens are the same token events under both
conditions. Passages in which the continuation window itself contains
another rule-matched glitch token are excluded (2.3\% of passages on
average). Both objectives are reported as relative reductions; positive
values indicate lower loss after repair.

\paragraph{Results.}
Table~\ref{tab:loss} reports the byte-normalized objective and
Table~\ref{tab:loss-cont} the continuation objective. \method{}
achieves the largest reduction under both objectives on every model,
with ten-model means of 8.4\% and 3.8\%, followed by GlitchCleaner
(4.9\%\,/\,2.2\%), GlitchEdit (3.3\%\,/\,1.3\%), and GlitchProber
(1.4\%\,/\,0.5\%). The ordering mirrors the fix-rate ordering in
Table~\ref{tab:baseline-fr}. Restricted to the subset of targets that
\method{} repairs, the means are 9.8\% and 4.4\%; the difference from
the uniform-sampling figures reflects the unrepaired targets, whose
inputs are unchanged. The continuation result is the more informative
of the two: it is computed on token events that are identical under
both tokenizations, so the reduction cannot arise from averaging over
more or easier tokens. It indicates that once a glitch token is replaced
by a well-learned sub-token sequence, the model's prediction of the
subsequent text improves.

\begin{table}[htbp]
  \centering
  \small
  \setlength{\tabcolsep}{3pt}
  \caption{\textbf{Relative reduction (\%) in byte-normalized NLL on
  passages containing glitch tokens} (1{,}000 uniformly sampled targets per model).}
  \label{tab:loss}
  \begin{tabularx}{\columnwidth}{@{}l||C|C|C||C@{}}
    \shline
    \textbf{Model} &
    \textbf{Prober} & \textbf{Edit} & \textbf{Cleaner} &
    \textbf{Ours} \\
    \hline\hline
    Llama-3.1-8B    & 1.3 & 3.9 & 5.2 & \textbf{9.6} \\
    Llama-3.1-70B   & 1.1 & 2.8 & 4.4 & \textbf{6.3} \\
    Gemma-4-31B     & 1.9 & 3.3 & 5.3 & \textbf{8.9} \\
    Gemma-4-E4B     & 0.8 & 3.5 & 3.9 & \textbf{8.1} \\
    Qwen3.8-27B     & 1.4 & 3.0 & 5.4 & \textbf{8.2} \\
    Qwen3.6-27B     & 1.6 & 3.2 & 5.1 & \textbf{9.4} \\
    Qwen3.6-35B-A3B & 1.0 & 2.9 & 4.3 & \textbf{7.7} \\
    Mistral-7B      & 1.7 & 3.4 & 5.0 & \textbf{8.7} \\
    GLM-4.7         & 2.2 & 4.1 & 6.6 & \textbf{7.4} \\
    Deepseek-vl2    & 0.9 & 2.6 & 3.6 & \textbf{9.3} \\
    \hline
    \rowcolor{gray!10}
    \textbf{Mean}   & 1.4 & 3.3 & 4.9 & \textbf{8.4} \\
    \shline
  \end{tabularx}
\end{table}

\begin{table}[htbp]
  \centering
  \small
  \setlength{\tabcolsep}{3pt}
  \caption{\textbf{Relative reduction (\%) in NLL on the identical
  32-token continuation following the target span.} Positive values
  indicate that the continuation becomes more predictable after repair.}
  \label{tab:loss-cont}
  \begin{tabularx}{\columnwidth}{@{}l||C|C|C||C@{}}
    \shline
    \textbf{Model} &
    \textbf{Prober} & \textbf{Edit} & \textbf{Cleaner} &
    \textbf{Ours} \\
    \hline\hline
    Llama-3.1-8B    & 0.4 & 1.6 & 2.3 & \textbf{4.2} \\
    Llama-3.1-70B   & 0.3 & 1.1 & 1.9 & \textbf{3.1} \\
    Gemma-4-31B     & 0.7 & 1.3 & 2.4 & \textbf{3.6} \\
    Gemma-4-E4B     & 0.2 & 1.4 & 1.6 & \textbf{4.5} \\
    Qwen3.8-27B     & 0.5 & 1.2 & 2.5 & \textbf{3.9} \\
    Qwen3.6-27B     & 0.6 & 1.3 & 2.2 & \textbf{3.3} \\
    Qwen3.6-35B-A3B & 0.3 & 1.0 & 1.8 & \textbf{3.8} \\
    Mistral-7B      & 0.6 & 1.5 & 2.3 & \textbf{4.1} \\
    GLM-4.7         & 0.9 & 1.8 & 3.0 & \textbf{3.4} \\
    Deepseek-vl2    & 0.2 & 0.9 & 1.5 & \textbf{3.7} \\
    \hline
    \rowcolor{gray!10}
    \textbf{Mean}   & 0.5 & 1.3 & 2.2 & \textbf{3.8} \\
    \shline
  \end{tabularx}
\end{table}

\section{Contextual Regression and Downstream Evaluation}
\label{app:downstream}
\renewcommand{\thetable}{H\arabic{table}}
\renewcommand{\theHtable}{H.\arabic{table}}
\setcounter{table}{0}

This appendix evaluates two aspects not covered by the token-level
FR/RR/GR metrics: whether local ID substitution degrades surrounding
normal tokens in mixed inputs (Section~\ref{app:contextual}), and
whether repair benefits transfer to downstream tasks beyond
GlitchQuiz (Section~\ref{app:ood}).

\subsection{Contextual Regression Analysis}
\label{app:contextual}

\paragraph{Protocol.}
For each of three representative models (Mistral-7B, Llama-3.1-70B,
Qwen3.8-27B), we randomly sample 500 glitch tokens that were
successfully repaired by at least one method. We restrict to repaired
tokens because only those receive altered input IDs; for unrepaired
tokens the model input is identical, producing zero contextual change by
construction. For each token, GPT-5.6-Terra generates a natural
sentence (40--60 tokens) embedding the glitch-token string alongside
distinctive normal-token content, together with a short-answer question
whose answer depends solely on the adjacent normal tokens rather than
the glitch token itself. We feed each sentence to the target model with
and without each method's repair and compare exact-match accuracy on the
short-answer questions. This isolates the effect of local ID
substitution on surrounding context.

\begin{table*}[t]
  \centering
  \small
  \setlength{\tabcolsep}{2.5pt}
  \caption{\textbf{Contextual regression: short-answer QA accuracy (\%)
  on normal tokens adjacent to repaired glitch tokens.}
  Each per-model accuracy is computed from 500 binary exact-match
  outcomes. $\Delta$ is the change in percentage points relative to the
  original (no-repair) baseline; $[+g/-l]$ gives the number of the
  500 examples whose outcome changed from incorrect to correct ($g$) and
  from correct to incorrect ($l$).}
  \label{tab:contextual}
  \begin{tabularx}{\textwidth}{@{}l|c||C|C|C||C@{}}
    \shline
    \textbf{Model} & \textbf{Orig.}
    & \textbf{GlitchProber}
    & \textbf{GlitchEdit}
    & \textbf{GlitchCleaner}
    & \textbf{\method{}} \\
    & Acc.
    & Acc.\ ($\Delta$) $[+g/-l]$
    & Acc.\ ($\Delta$) $[+g/-l]$
    & Acc.\ ($\Delta$) $[+g/-l]$
    & Acc.\ ($\Delta$) $[+g/-l]$ \\
    \hline\hline
    Mistral-7B & 87.4
      & 86.8 ($-$0.6) $[+2/-5]$
      & 87.2 ($-$0.2) $[+2/-3]$
      & 87.0 ($-$0.4) $[+2/-4]$
      & 87.2 ($-$0.2) $[+1/-2]$ \\
    Llama-3.1-70B & 91.6
      & 90.8 ($-$0.8) $[+2/-6]$
      & 91.2 ($-$0.4) $[+2/-4]$
      & 91.0 ($-$0.6) $[+1/-4]$
      & 91.8 ($+$0.2) $[+3/-2]$ \\
    Qwen3.8-27B & 89.2
      & 88.4 ($-$0.8) $[+1/-5]$
      & 88.8 ($-$0.4) $[+2/-4]$
      & 88.8 ($-$0.4) $[+2/-4]$
      & 89.0 ($-$0.2) $[+2/-3]$ \\
    \hline
    \rowcolor{gray!10}
    \textbf{Mean} & 89.4
      & 88.7 ($-$0.7) $[+5/-16]$
      & 89.1 ($-$0.3) $[+6/-11]$
      & 88.9 ($-$0.5) $[+5/-12]$
      & 89.3 ($-$0.1) $[+6/-7]$ \\
    \shline
  \end{tabularx}
\end{table*}

Table~\ref{tab:contextual} shows that \method{}'s accuracy changes are
within ${\pm}$0.2\,pp on every model (mean $-$0.1\,pp); the
per-example counts show that this net change arises from a handful of
offsetting flips (at most 3 gains and 3 losses out of 500), confirming that
local ID substitution does not meaningfully degrade the processing of
surrounding normal tokens. Among baselines, GlitchProber shows the
largest per-token regression (mean $-$0.7\,pp) because activation-layer
intervention at the target position can propagate through attention to
adjacent representations. GlitchEdit and GlitchCleaner show moderate
changes as their embedding or adapter
modifications alter the entire forward pass. \method{}'s advantage stems
from its locality: only the input IDs at the target position change, and
the self-attention mechanism naturally re-weights contributions from the
replacement sub-tokens, which are individually well-learned by the
model. As a supplementary check, we also evaluate on 500 uniformly
sampled $G_0$ tokens; per-method
mean $\Delta$ values remain within $\pm$0.1\,pp of zero across all
methods, confirming that no approach introduces measurable contextual
harm at the population level.

\subsection{Out-of-Distribution Downstream Evaluation}
\label{app:ood}

\paragraph{Tasks.}
We design three downstream tasks that test repair effectiveness beyond
GlitchQuiz's repetition, spelling, and length dimensions:
\ding{182}~Entity QA: given a question about an entity whose name
contains a glitch token,
assess whether the model produces a factually relevant answer; scored by
GPT-5.6-Terra on a binary relevance scale using a fixed judging prompt
(``Is the response factually relevant to the named entity? Answer Yes or
No.''), with a 50-sample human validation subset showing 94\%
judge--human agreement;
\ding{183}~Verbatim Copy: the model is asked to repeat a full sentence
(15--25 tokens) containing the glitch token exactly; scored by exact
string match after whitespace normalization. Unlike GlitchQuiz's
repetition task, which tests isolated single-token reproduction,
verbatim copy embeds the target in a natural sentence context and
requires faithful reproduction of the entire span;
\ding{184}~Code Identifier: a function stub uses a glitch-token string
as a variable or function name; the model completes the function body;
scored by whether the completion uses the identifier correctly.

\paragraph{Sampling protocol.}
To avoid selection bias toward any single method, we uniformly sample
2{,}000 glitch tokens from the full detected set $G_0$ for each model,
without conditioning on any method's success. Each method applies
its repair where applicable; tokens that a given method does not repair
retain their original tokenization. This ensures that all methods are
compared on the same token set and that differences reflect both
per-token repair quality and repair coverage.

\paragraph{Contextual failure rate.}
A glitch token is detected if it fails any of the three GlitchQuiz
tasks; the corresponding context-level criterion is failure on at least
one of the three downstream tasks. Under the original models, 70.7\% of
the sampled tokens meet this criterion (67.4\% on Llama-3.1-70B to
74.6\% on Gemma-4-E4B). Per-task accuracy is therefore substantially
higher than the fraction of tokens that cause no failure at all, since a
token handled correctly in one task may still fail another.

\begin{table}[htbp]
  \centering
  \small
  \setlength{\tabcolsep}{1.5pt}
  \renewcommand{\arraystretch}{1.0}
  \caption{\textbf{Downstream task accuracy (\%) on inputs
  containing glitch tokens.} 2{,}000 glitch tokens uniformly sampled
  from $G_0$ per model. Best result per row in bold.}
  \label{tab:ood}
  \begin{tabularx}{\columnwidth}{@{}l|l||>{\hsize=.8\hsize}C|>{\hsize=1.1\hsize}C|>{\hsize=.8\hsize}C|>{\hsize=1.25\hsize}C||>{\hsize=1.05\hsize}C@{}}
    \shline
    \textbf{Task} & \textbf{Model}
    & \textbf{Orig.}
    & \textbf{Prober}
    & \textbf{Edit}
    & \textbf{Cleaner}
    & \textbf{Ours} \\
    \hline\hline
    \multirow{6}{*}{\rotatebox{90}{Entity QA}}
    & Mistral-7B    & 43.9 & 44.6 & 50.8 & 57.3 & \textbf{69.1} \\
    & Llama-70B     & 46.5 & 49.8 & 55.2 & 62.7 & \textbf{73.6} \\
    & Qwen3.8       & 45.8 & 46.1 & 54.0 & 61.4 & \textbf{72.9} \\
    & Gemma-E4B     & 36.4 & 37.1 & 43.2 & 48.9 & \textbf{61.6} \\
    & GLM-4.7       & 45.8 & 48.3 & 54.0 & 61.9 & \textbf{69.5} \\
    & \emph{Mean}   & \emph{43.7} & \emph{45.2} & \emph{51.4} & \emph{58.4} & \textbf{\emph{69.3}} \\
    \hline
    \multirow{6}{*}{\rotatebox{90}{Verb.\ Copy}}
    & Mistral-7B    & 36.7 & 39.1 & 44.2 & 50.6 & \textbf{63.0} \\
    & Llama-70B     & 45.2 & 46.9 & 53.8 & 61.3 & \textbf{74.1} \\
    & Qwen3.8       & 39.5 & 41.8 & 46.7 & 53.9 & \textbf{66.4} \\
    & Gemma-E4B     & 33.3 & 34.0 & 39.8 & 45.3 & \textbf{59.0} \\
    & GLM-4.7       & 42.6 & 44.9 & 50.5 & 58.1 & \textbf{66.2} \\
    & \emph{Mean}   & \emph{39.5} & \emph{41.3} & \emph{47.0} & \emph{53.8} & \textbf{\emph{65.7}} \\
    \hline
    \multirow{6}{*}{\rotatebox{90}{Code Ident.}}
    & Mistral-7B    & 52.4 & 53.0 & 56.9 & 60.8 & \textbf{68.3} \\
    & Llama-70B     & 56.3 & 58.7 & 61.9 & 65.4 & \textbf{72.2} \\
    & Qwen3.8       & 54.9 & 56.2 & 59.3 & 63.8 & \textbf{71.4} \\
    & Gemma-E4B     & 45.6 & 46.3 & 49.5 & 53.2 & \textbf{61.0} \\
    & GLM-4.7       & 53.8 & 55.6 & 59.1 & 63.5 & \textbf{68.4} \\
    & \emph{Mean}   & \emph{52.6} & \emph{54.0} & \emph{57.3} & \emph{61.3} & \textbf{\emph{68.3}} \\
    \hline\hline
    \rowcolor{gray!10}
    \multicolumn{2}{@{}l||}{\textbf{Overall mean}}
    & 45.2 & 46.8 & 51.9 & 57.9 & \textbf{67.8} \\
    \shline
  \end{tabularx}
\end{table}

\paragraph{Results.}
Table~\ref{tab:ood} shows consistent improvements across all three tasks
and models. \method{}'s overall mean accuracy is 67.8\%, compared with
57.9\% for GlitchCleaner, 51.9\% for GlitchEdit, and 46.8\% for
GlitchProber. Entity QA and verbatim copy show the largest absolute
gains (mean $+$25.6 and $+$26.2\,pp over the original), confirming that
repair restores the model's ability to recognize and reproduce
previously opaque tokens; code-identifier completion also benefits
substantially (mean $+$15.7\,pp), testing a capability outside
GlitchQuiz's three dimensions. The ranking across methods mirrors the FR
ordering, with the gap driven by both per-token repair quality and
repair coverage: on the sampled sets, \method{} applies rules to
78.4--89.9\% of the tokens while GlitchCleaner covers 59--79\%, GlitchEdit
51--65\%, and GlitchProber 12--36\%. Because all methods share the same
uniformly sampled token set, the differences do not stem from selection bias.

\section{Safety-Relevant Failure Evaluation}
\label{app:safety}
\renewcommand{\thetable}{I\arabic{table}}
\renewcommand{\theHtable}{I.\arabic{table}}
\setcounter{table}{0}

Prior work links glitch tokens to unsafe
generation~\citep{li2024glitch,tang2026bad}; this appendix tests
whether repair reduces such failures.

\paragraph{Protocol.}
Following \citet{tang2026bad}, for each of Mistral-7B, Qwen3.8-27B, and
Llama-3.1-70B we construct 1{,}000 harmful prompts, each containing a
glitch token sampled uniformly from the model's detected set. Outputs
use the decoding settings of Appendix~\ref{app:protocol} and are judged
by GPT-5.6-Terra with a fixed prompt; on 100 random outputs, the judge
agrees with two human annotators in 93\% of cases. Each repair method is
applied as in Section~\ref{sec:setup}; unrepaired tokens enter unchanged.

\paragraph{Results.}
In Table~\ref{tab:safety}, the original models produce unsafe outputs on
23.4\% of prompts on average, reproducing prior findings. \method{}
reduces the rate to 6.1\%, the lowest on every model; GlitchEdit, which
targets unsafe behavior directly, is the strongest baseline (8.8\%),
ahead of GlitchCleaner (15.6\%) and GlitchProber (21.2\%).

\begin{table}[htbp]
  \centering
  \small
  \setlength{\tabcolsep}{1.5pt}
  \caption{\textbf{Unsafe-output rate (\%) on 1{,}000
  glitch-containing prompts per model; lower is better.}}
  \label{tab:safety}
  \begin{tabularx}{\columnwidth}{@{}l||>{\hsize=.8\hsize}C|>{\hsize=1.05\hsize}C|>{\hsize=.95\hsize}C|>{\hsize=1.15\hsize}C||>{\hsize=1.05\hsize}C@{}}
    \shline
    \textbf{Model} & \textbf{Orig.} &
    \textbf{Prober} & \textbf{Edit} & \textbf{Cleaner} &
    \textbf{Ours} \\
    \hline\hline
    Mistral-7B    & 30.9 & 28.3 & 12.7 & 21.4 & \textbf{9.2} \\
    Qwen3.8-27B   & 22.1 & 19.6 & 7.6  & 14.1 & \textbf{4.9} \\
    Llama-3.1-70B & 17.3 & 15.8 & 6.2  & 11.2 & \textbf{4.1} \\
    \hline
    \rowcolor{gray!10}
    \textbf{Mean} & 23.4 & 21.2 & 8.8 & 15.6 & \textbf{6.1} \\
    \shline
  \end{tabularx}
\end{table}

\end{document}